%% file: root.tex
\documentclass[10pt]{article}
\usepackage[preprint]{style}

\input{biblatex}
\usepackage{hyperref}
\input{commands}

\begin{document}

\title{Deep Grey-Box Modeling With Adaptive Data-Driven Models\\Toward Trustworthy Estimation of Theory-Driven Models}

\author{\name Naoya Takeishi \email naoya.takeishi@hesge.ch \\
      \addr Geneva School of Business Administration\\University of Applied Sciences and Arts Western Switzerland (HES-SO)
      \AND
      \name Alexandros Kalousis \email alexandros.kalousis@hesge.ch \\
      \addr Geneva School of Business Administration\\University of Applied Sciences and Arts Western Switzerland (HES-SO)
}

\maketitle

\begin{abstract}
The combination of deep neural nets and theory-driven models, which we call \emph{deep grey-box modeling}, can be inherently interpretable to some extent thanks to the theory backbone. Deep grey-box models are usually learned with a regularized risk minimization to prevent a theory-driven part from being overwritten and ignored by a deep neural net. However, an estimation of the theory-driven part obtained by uncritically optimizing a regularizer can hardly be trustworthy when we are not sure what regularizer is suitable for the given data, which may harm the interpretability. Toward a trustworthy estimation of the theory-driven part, we should analyze regularizers' behavior to compare different candidates and to justify a specific choice. In this paper, we present a framework that enables us to analyze a regularizer's behavior empirically with a slight change in the neural net's architecture and the training objective.
\end{abstract}

\input{01_introduction}

\input{02_dgm}

\input{03_proposed}

\input{04_relatedwork}

\input{05_experiment}

\input{06_conclusion}

\subsubsection*{Acknowledgements}
This work was supported by the Innosuisse project \emph{Industrial artificial intelligence for intelligent machines and manufacturing digitalization} (39453.1 IP-ICT) and the Swiss National Science Foundation Sinergia project \emph{Modeling pathological gait resulting from motor impairments} (CRSII5\_177179).

\setlength\bibitemsep{0.55\baselineskip}
{
% \small
\printbibliography
}

\input{07_appendix}

\end{document}

%% file: biblatex.tex
\usepackage{xpatch}
\usepackage[backend=biber,natbib=true,style=authoryear-icomp,giveninits=true,maxbibnames=99,maxcitenames=2,uniquelist=false,uniquename=init]{biblatex}

\bibliography{main}

\xpatchbibmacro{date+extrayear}{%
    \printtext[parens]%
}{%
    \setunit{\addperiod\space}%
    \printtext%
}{}{}

\DeclareFieldFormat[inproceedings]{title}{#1}   
\DeclareFieldFormat[incollection]{title}{#1}    
\DeclareFieldFormat[article]{title}{#1}
\DeclareFieldFormat[misc]{title}{#1}

\DefineBibliographyStrings{english}{%
    page             = {\ifbibliography{}{p\adddot}},
    pages            = {\ifbibliography{}{pp\adddot}},
} 

\DeclareNameAlias{sortname}{family-given}

%% file: commands.tex
\usepackage{amsmath}
\usepackage{amssymb}
\usepackage{mathtools}
\usepackage{bm}
\usepackage{amsthm}
\usepackage[capitalise]{cleveref}
\usepackage{wrapfig}
\usepackage{makecell}
\usepackage{booktabs}

\usepackage{array}
\newcolumntype{P}[1]{>{\centering\arraybackslash}p{#1}}
\newcolumntype{M}[1]{>{\centering\arraybackslash}m{#1}}

\usepackage{enumitem}
\setlist[itemize]{leftmargin=15pt,topsep=0pt,itemsep=0pt}

\newtheorem{proposition}{Proposition}
\theoremstyle{definition}
\newtheorem{assumption}{Assumption}
\theoremstyle{remark}
\newtheorem{remark}{Remark}

\usepackage{pgfplots}
\pgfplotsset{compat=1.9}
\usepgfplotslibrary{external}
\pgfplotsset{every axis/.append style={
    label style={font=\small},
    tick label style={font=\scriptsize},
}}
\definecolor{myred}{RGB}{255,31,91}
\definecolor{mygreen}{RGB}{0,205,108}
\definecolor{myblue}{RGB}{0,154,222}
\definecolor{mypurple}{RGB}{175,88,186}
\definecolor{myyellow}{RGB}{255,198,30}
\definecolor{myorange}{RGB}{242,133,34}
\definecolor{mygray}{RGB}{160,177,186}
\definecolor{mybrown}{RGB}{166,118,29}
\definecolor{myred2}{RGB}{233,0,45}
\definecolor{mygreen2}{RGB}{0,176,0}

\newcommand{\T}{{\mathrm{T}}}
\newcommand{\D}{{\mathrm{D}}}

%% file: 01_introduction.tex
\section{INTRODUCTION}

Grey-box modeling in general refers to the combination of theory-driven structures and data-driven components \citep[see, e.g.,][]{sohlbergGreyBoxModelling2008}.
In this paper, we are interested in combining theory-driven models such as mathematical models of physical phenomena, and data-driven, machine-learning models.
For example, for regressing from $x$ to $y$, we are interested in models like
\begin{equation*}
    y = f_\T(x; \theta_\T) + f_\D(x; \theta_D) + e,
\end{equation*}
where $f_\T$ and $f_\D$ denote theory- and data-driven models parameterized by $\theta_\T$ and $\theta_\D$, respectively, and $e$ is noise.
A more general model will appear in \cref{dgm} and thereafter.
We discuss how we should (or should not) cast the estimation problem of a grey-box model's parameters.
% \footnote{Choosing the form of $f_\T$ and $f_\D$ (i.e., model selection) is also an interesting problem but is out of the scope of this paper.}

We argue over the interpretability of grey-box models.
They can be interpretable because a part of the model lies on a theory backbone and often has a small number of parameters.
They are valuable tools to glimpse insights from data on which our theory is essentially incomplete.
However, as we discuss later, interpretability is not a free lunch, and we need to pay attention to how we perform the estimation to secure the trustworthiness of the interpretation.

We are interested in cases where the data-driven model, $f_\D$, is a deep neural network.
In such a case, care must be taken for the theory-based model, $f_\T$, not to be overwritten and ignored by $f_\D$ due to the expressive power of the latter.
Such models have been studied recently \citep[][see also \cref{related}]{yinAugmentingPhysicalModels2021,takeishiPhysicsintegratedVariationalAutoencoders2021,qianIntegratingExpertODEs2021,wehenkelRobustHybridLearning2022}.
Deep grey-box models are typically learned by solving optimization problems like
\begin{equation*}
    \underset{\theta_\T, \theta_\D}{\text{minimize}} ~ \mathcal{L} + \lambda \mathcal{R},
\end{equation*}
where $\mathcal{L}$ is a prediction loss, and $\mathcal{R}$ is a regularizer needed for preventing $f_\T$ from being ignored.
Although this optimization may result in a model with good prediction performance, we cannot judge if such an estimator is worth being interpreted.
There can be multiple parameter values that achieve similar performance, and the optimization's solution tells us nothing about the property of $\mathcal{R}$, so it provides only insufficient information to justify a choice of $\theta_\T$ to be interpreted.
For the same reason, we cannot compare different $\mathcal{R}$s solely based on the optimization's solution.
Instead, we should analyze the behavior of $\mathcal{R}$, at least empirically, toward obtaining a trustworthy estimation and its interpretation.

Our idea is simple: We withhold the estimation of $\theta_\T$, at least in the first trial of data analysis.
To this end, we learn $f_\D$ adaptively with realizations of $\theta_\T$, essentially marginalizing out $\theta_\T$.
Such a slight change in the formulation allows us to analyze $\mathcal{R}$ empirically by examining its landscape without any need for re-training.
% Making $f_\D$ work adaptively to $\theta_\T$ is only possible when $f_\D$ is a strong function approximator, and $\theta_\T$ is relatively low-dimensional; this is the exact setting of deep grey-box models.
In this paper, we take up the aforementioned argument about the estimation's trustworthiness for discussion, formulate the above idea, and conduct an empirical investigation of its effectiveness.

A useful byproduct of the proposed formulation is that the optimization of $\mathcal{L}$ and $\mathcal{R}$ can now be decoupled.
It allows us to use different optimizers for the two objectives, as well as to use unlabeled test data to optimize $\mathcal{R}$.
Moreover, we found that such a decoupled optimization makes the optimization much less sensitive to the hyperparameter, $\lambda$.

%% file: 02_dgm.tex
\section{PRELIMINARY}
\label{dgm}

\subsection{Definition}
\label{dgm:def}

We define \emph{deep grey-box models} as compositions of theory-driven models and data-driven models, with the latter being deep neural networks.
For the sake of discussion, we suppose regression problems where $y \in \mathcal{Y}$ is to be predicted from $x \in \mathcal{X}$, though the extension to other problems is straightforward.
We denote such a model in general by
\begin{equation}\label{eq:model}
    y = \mathcal{C}(f_\T, f_\D; x),
\end{equation}
where $\mathcal{C}$ is a functional that takes the two types of functions and an input variable as arguments.
The two functions, $f_\T$ and $f_\D$, are a theory-driven model and a deep neural network, with unknown parameters $\theta_\T \in \Theta_\T$ and $\theta_\D \in \Theta_\D$, respectively.
We may write $f_\T(x;\theta_T)$ to manifest $f_\T$'s dependency on $\theta_\T$ or write simply $f_\T(x)$ though it still depends on $\theta_\T$ (and analogously for $f_\D$ and $\theta_\D$).
Note that not only $f_\D$ but also $f_\T$ may have unknown parameters to be inferred.
We usually expect $\dim\theta_\T \ll \dim\theta_\D$.
The functional, $\mathcal{C}$, evaluates $f_\T$ and $f_\D$ with a given $(\theta_\T,\theta_D,x)$ and then mixes up their outputs to give the final output of the model.

We try to keep the generality of $\mathcal{C}$; it may include general function compositions and their arbitrary transformations:
\begin{equation*}
    \mathcal{C}(f_\T, f_\D; x) = \operatorname{SomeTransformation}[ f_\D(f_\T(x), x) ].
\end{equation*}
Meanwhile, one of the most prevailing forms of $\mathcal{C}$ in the literature is the additive grey-box ODEs like:
\begin{equation*}
    \mathcal{C}(f_\T, f_\D; x) = \operatorname{ODESolve} [ \dot{s} = f_\T(s) + f_\D(s) \mid s_0=x ],
\end{equation*}
where $s$ is the state variable of the dynamics, $s_0$ is the initial condition, and $\operatorname{ODESolve}$ denotes an operation that numerically solves initial value problems.
Such \emph{grey-box} (ordinary or partial) differential equations have been studied by researchers such as \citet{sasakiNeuralGrayboxIdentification2019,yinAugmentingPhysicalModels2021,takeishiPhysicsintegratedVariationalAutoencoders2021,qianIntegratingExpertODEs2021}.

We are particularly interested in cases where $\mathcal{C}$ inherits the expressive power of $f_\D$ as a function approximator.
This is \emph{not} the case, for example, when $\mathcal{C}$ only contains compositions such as $f_\T(f_\D(x), x)$, that is, $f_\T$ is ``outside'' $f_\D$ \citep[e.g.,][]{raissiPhysicsinformedNeuralNetworks2019,arikInterpretableSequenceLearning2020,schnellHalfinverseGradientsPhysical2022}.
Estimation of such models is less challenging because $f_\T$ cannot be ignored by construction.
In contrast, we address more difficult cases where $f_\T$ is ``inside'' $f_\D$, for which we should be careful so that $f_\T$ is not overwritten and ignored by $f_\D$.
We put the following assumptions on the model:
\begin{assumption}\label{asmp:fD}
    $f_\D\colon \mathcal{X} \!\to\! \mathcal{Y}$ is a universal function approximator; for any $\epsilon > 0$ and a continuous function $g\colon \mathcal{X} \to \mathcal{Y}$, there exists $\theta_\D \in \Theta_\D$ satisfying $\sup_{x \in S_X} \Vert f_\D(x; \theta_\D) - g(x) \Vert < \epsilon$, where $S_X \subset \mathcal{X}$ is some compact set.
\end{assumption}
\begin{assumption}\label{asmp:C}
    $\mathcal{C}(f_\T, f_\D; \cdot)\colon \mathcal{X} \to \mathcal{Y}$ is also a universal function approximator; that is, for any $\epsilon>0$, $\theta_\T\in\Theta_\T$, and a continuous function $g'\colon \mathcal{X} \to \mathcal{Y}$, there exists $\theta_\D\in\Theta_\D$ satisfying $\sup_{x \in S_X} \Vert \mathcal{C}(f_\T, f_\D; x) - g'(x) \Vert < \epsilon$.
\end{assumption}
\begin{remark}
    We assume the universal approximation property just for rigorously construct the discussion.
    Even without the universal approximation property, as long as $f_\D$ and $\mathcal{C}$ are much more expressive than $f_\T$, discussions below would approximately hold in practice.
\end{remark}
\begin{assumption}\label{asmp:Lipschitz_f}
    $f_\T$ and $f_\D$ are Lipschitz continuous with regard to $(x,\theta_\T)$ and $(x,\theta_\D)$, respectively.
\end{assumption}

\subsection{Why Grey-box?}

Deep grey-box models are powerful function approximators with a certain level of inherent interpretability owing to $f_T$, a human-understandable model with a theory as a backbone.
A typical use case would be to estimate a grey-box model on data on which our theory is essentially incomplete and inspect the estimated model to glimpse insights, e.g., when the incomplete theory is correct or not, how the missing part approximated by $f_\D$ behaves, and so on.

Deep grey-box models can also be advantageous in generalization capability and robustness to extrapolation, as reported empirically so far \citep[][]{qianIntegratingExpertODEs2021,yinAugmentingPhysicalModels2021,takeishiPhysicsintegratedVariationalAutoencoders2021,wehenkelRobustHybridLearning2022}.
It is natural to expect such improvements because the presence of $f_\T$ in the model would reduce the sample complexity of the learning problem, and $f_\T$ is supposed to work well in the out-of-data regime (in other words, it is a requirement for a model to be regarded as theory-driven).
However, rigorous analysis of generalization is challenging for models involving deep neural nets.
Anyway, we do not touch on such performance aspects of deep grey-box models given the previous studies, so the comparison to non-grey-box models is out of the paper's scope.

\subsection{Empirical Risk Minimization Cannot Select \texorpdfstring{$\theta_T$}{θT}}
\label{dgm:difficulty}

There is a natural consequence of deep grey-box modeling; the theory-driven model's parameter, $\theta_\T$, cannot be chosen solely by minimizing an empirical risk of prediction.
For example, suppose we learn $\mathcal{C}(f_\T, f_\D;x) = f_\T(x) + f_\D(x)$ by minimizing the mean squared error, $\mathcal{L} = \Vert y - (f_\T(x) + f_\D(x)) \Vert_2^2$.
The empirical risk can be minimized to a similar extent for \emph{any} $\theta_\T \in \Theta_\T$ because $f_\D$, a deep neural net, can, it alone, approximate any function on the training set (as assumed in \cref{asmp:fD}) and thus also the function $y-f_\T(x)$.
We formally state this fact as follows:
\begin{proposition}\label{prop:erm}
    Let $S = \{(x_1,y_1),\dots,(x_n,y_n)\}$ be a training set.
    Let $\mathcal{L}_{(x,y)}(\theta_\T, \theta_\D)$ be a Lipschitz continuous loss function between the prediction (i.e., the value of $\mathcal{C}(f_\T, f_\D; x)$) and the target (i.e., $y$).
    Let $\mathcal{L}_S(\theta_\T,\theta_\D)=\sum_{(x,y) \in S} \mathcal{L}_{(x,y)}(\theta_\T, \theta_\D)$ be the empirical risk on the training set.
    Suppose that Assumptions~\ref{asmp:fD}--\ref{asmp:Lipschitz_f} hold.
    Then, for any $\epsilon'>0$, $\theta_\D \in \Theta_\D$, and $\theta_\T, \theta'_\T \in \Theta_\T$ where $\theta_\T \neq \theta'_\T$, there exists $\theta'_\D \in \Theta_\D$ that satisfies
    \begin{equation}\label{eq:erm}
        \left\vert \mathcal{L}_S(\theta_\T, \theta_\D) - \mathcal{L}_S(\theta'_\T, \theta'_\D) \right\vert < \epsilon'.
    \end{equation}
\end{proposition}
\begin{proof}
    From the assumptions, for any $\epsilon > 0$, $\theta_\D \in \Theta_D$, and $\theta_\T \neq \theta'_\T \in \Theta_\T$, there exists $\theta'_\D \in \Theta_D$ that satisfies $\sup_{x \in \{x_1,\dots,x_n\}} \Vert \mathcal{C}(f_\T, f_\D; x) - \mathcal{C}(f'_\T, f'_\D; x) \Vert < \epsilon$, where $f'_i$ is parameterized by $\theta'_i$ for $i=\T,\D$.
    Since $\mathcal{L}$ is Lipschitz continuous, $\sup_{x,y \in S} \vert \mathcal{L}_{(x,y)}(\theta_\T,\theta_\D) - \mathcal{L}_{(x,y)}(\theta'_\T,\theta'_\D) \vert < K \epsilon$ with where $K$ is $\mathcal{L}$'s Lipschitz constant.
    Therefore, with $\epsilon' \coloneqq \vert S \vert K \epsilon$, $\vert \mathcal{L}_S(\theta_\T, \theta_\D) - \mathcal{L}_S(\theta'_\T, \theta'_\D) \vert < \epsilon'$.
\end{proof}

\subsection{Regularized Risk Minimization}
\label{dgm:existing}

\cref{prop:erm} states that any $\theta_\T \in \Theta_\T$ can be equally likely solely under the empirical risk.
It necessitates regularizing the problem; we should optimize $\mathcal{L}_S + \lambda \mathcal{R}$ instead, where $\lambda \geq 0$ is a regularization hyperparameter, and $\mathcal{R}$ is some regularizer that should reflect our inductive biases on how we should combine the theory- and data-driven models.
Let us, for example, consider the linear combination case, $\mathcal{C}(f_\T, f_\D; x) = f_\T(x) + f_\D(x)$.
One of the common ways of thinking is that $f_\T$ should as accurately explain the $x$--$y$ relation as possible, and $f_\D$ should have the least possible effect.
This idea can be operationalized by defining $\mathcal{R} = \Vert f_\D \Vert$, where the norm is the function's norm.
Though such an $\mathcal{R}$ has been a popular choice, it is not the only possibility.
For example, when one wants the two models' output to be uncorrelated, one can use $\mathcal{R} = \vert \langle f_\T(x), f_\D(x) \rangle \vert$.\footnote{Suggesting specific $\mathcal{R}$ for each application or in general is out of the scope of this paper; on contrary, our proposal in \cref{method} is for cases where we cannot specify $\mathcal{R}$ \emph{a priori}.}
We assume that $\mathcal{R}$ depends only on $x$.
It is natural because the role of $\mathcal{R}$ is not to fit the $x$--$y$ relation.
We will recall this assumption, if necessary, by writing $\mathcal{R}_{S_X}$, where $S_X = \{x_1,\dots,x_n\}$ is the extract of $x$s from $S$.
We do not suppose, at least explicitly, any more specifications of $\mathcal{R}$ than this assumption.

The regularized estimation problem can be cast as follows:

\paragraph{Inductive Learning}

The simplest formulation is
\begin{equation}\label{eq:inductive}
    \theta_\T^*, \theta_\D^* = \arg \min_{\theta_\T, \theta_\D} \mathcal{L}_S(\theta_\T, \theta_\D) + \lambda \mathcal{R}_{S_X}(\theta_\T, \theta_\D).
\end{equation}
In this formulation, not only $\mathcal{L}$ but also $\mathcal{R}$ suffers a generalization gap.
Also, $\lambda$ needs to be tuned somehow.

\paragraph{Transductive Learning}

Since we assume that $\mathcal{R}$ only depends on $x$, it is reasonable to mention \emph{transductive learning} \citep{gammermanLearningTransduction1998}.
Let $S'_X$ be some set of $x$ that may include $S_X$ as a subset.
The idea is to minimize the unsupervised part of the objective, not on the training data $S_X$ but rather on $S'_X$;
\begin{equation}\label{eq:transductive}
    \theta_\T^*, \theta_\D^* = \arg \min_{\theta_\T, \theta_\D} \mathcal{L}_S(\theta_\T, \theta_\D) + \lambda \mathcal{R}_{S'_X}(\theta_\T, \theta_\D).
\end{equation}
The generalization gap disappears for $\mathcal{R}$ when (a subset of) $S'_X$ is the test set, but still $\lambda$ needs to be tuned.

%% file: 03_proposed.tex
\section{TOWARD TRUSTWORTHY ESTIMATION}
\label{method}

\subsection{Challenges of Deep Grey-box Model Estimation}

Deep grey-box models have been studied mainly in terms of empirical generalization and extrapolation capability  \citep{yinAugmentingPhysicalModels2021,takeishiPhysicsintegratedVariationalAutoencoders2021,qianIntegratingExpertODEs2021,wehenkelRobustHybridLearning2022}.
However, when the model's interpretation is concerned, the prediction performance does not speak a lot; there can be multiple parameter values that perform similarly (cf. \emph{Rashomon} sets), and we cannot judge which one we should interpret.
The solution of the optimization in \cref{eq:inductive} or \eqref{eq:transductive} tells us nothing about the analytical property of $\mathcal{R}$, so we can hardly understand the full picture of how the optimization selects $\theta_\T$.
Instead of uncritically optimizing the regularizer, $\mathcal{R}$, we should know the property of $\mathcal{R}$ in order to gain more information to explain the choice of $\theta_\T$ to be interpreted.
We would contrast the situation with, for example, the estimators of linear regression models, which have been extensively analyzed and thus are trustworthy in some sense.
We do not suggest analyzing our $\mathcal{R}$s analytically as it is too problem-dependent, but analyzing them at least empirically would help us make $\theta_\T$'s estimation more trustworthy.

The challenge due to not knowing $\mathcal{R}$'s property stands out more when we do not know what $\mathcal{R}$ is suitable for the given data and need to compare different candidate $\mathcal{R}$s, which is often the case as we do not know the whole data-generating process.
The point estimation via \cref{eq:inductive} or \eqref{eq:transductive} would not tell much about the goodness of $\mathcal{R}$, since different $\mathcal{R}$s could achieve similar prediction performance.
This viewpoint also supports the need for analyzing $\mathcal{R}$ at least empirically for gaining information to compare different $\mathcal{R}$s.

Another challenge, yet more technical, is the choice of the regularization hyperparameter, $\lambda$.
It can be tricky because, in \cref{eq:inductive} or \eqref{eq:transductive}, it controls two things at the same time: ``which $\theta_\T$ should be selected'' and ``how much $f_\D$ should be regularized.''
They are different problems if interrelated, and thus decoupling them would be beneficial.

\subsection{Proposed Formulation}

As we argued above, analyzing $\mathcal{R}$ empirically can be a useful first step toward a trustworthy estimation of $\theta_\T$.
More specifically, we aim to explore the landscape of $\mathcal{R}$, that is, to evaluate the values of $\mathcal{R}$ for different $\theta_\T$s.
The na\"ive way to do so is to re-run the optimization in \cref{eq:inductive} or \eqref{eq:transductive} for many times with different values of $\theta_\T$ fixed at each time, but it is inefficient even for a low-dimensional $\theta_\T$, and moreover, it does not allow to predict with a new value of $\theta_\T$ that was not exactly tried during such many training runs.
Instead, we suggest an approach to exploring $\mathcal{R}$'s landscape without the need for re-running the optimization.

At the core of our suggestion is ``marginalizing out'' $\theta_\T$ in the training phase.
To this end, we slightly modify $f_\D$ so that it works adaptively with different values of $\theta_\T$ and then minimize the objective, taking its expectation with regard to $\theta_\T$.
It should result in a model that can predict equally well given whatever $\theta_\T$ in some feasible region.
We first estimate only $\theta_\D$ while leaving $\theta_\T$ undetermined; thus the idea is similar to probabilistic inference where some variables are marginalized out.
The remainder of this section explains each step of the proposed formulation.
\Cref{fig:overview} depicts the overview of the formulation.

\begin{figure}
    \centering
    \includegraphics[clip,width=0.6\linewidth]{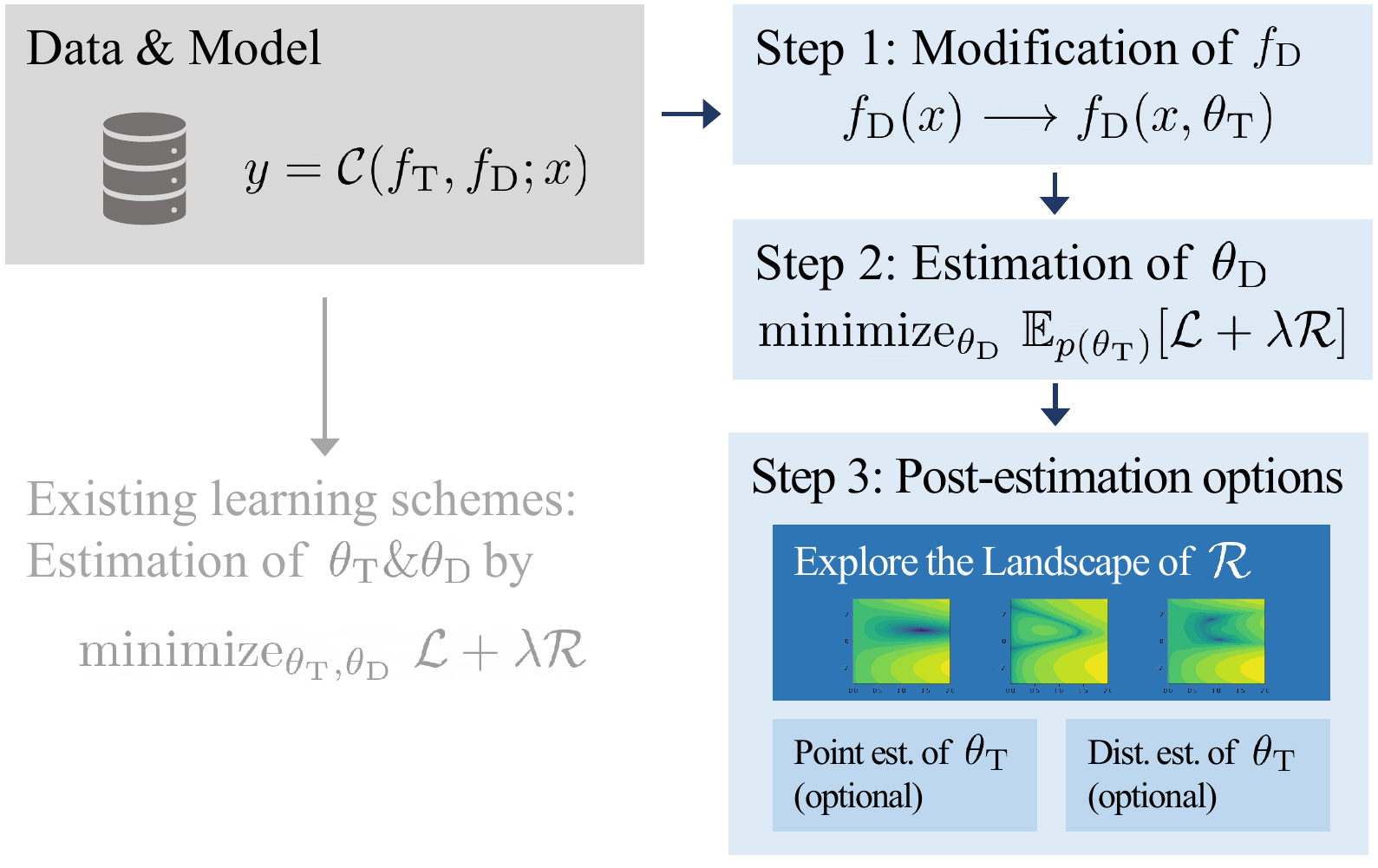}
    \caption{Overview of the proposed formulation.}
    \label{fig:overview}
\end{figure}

\subsubsection{Step 1: Modification of \texorpdfstring{$f_\D$}{fD}}
\label{method:step1}

% Recall that, in the definition of deep grey-box models, \cref{eq:model}, we did not specify the architecture and the argument of $f_\D$, a deep neural net.
We add two specifications of $f_\D$ to be used, without much loss of generality.
First, we suppose that $f_\D$ takes as arguments not only $x$ but also $\theta_\T$ (and possibly the value of $f_\T(x;\theta_\T)$).
Second, we define $f_\D$'s architecture such that $f_\D$ can adapt to different values of $\theta_\T$ given as an argument.
We can realize them with general techniques of conditional modeling; for example, by concatenating $\theta_\T$ to the original input; or by using hypernetworks that transform $\theta_\T$ to (a part of) $\theta_\D$.
In the experiments in \cref{exp}, we simply concatenate $\theta_\T$ and $f_\T(x;\theta_\T)$ to $x$ and feed them to $f_\D$ either with an extended input layer or with additional channels.
We may emphasize the dependency of $f_\D$ on $\theta_\T$ (and thus on $f_\T$) by denoting it as $f_\D(f_\T)$ in what follows.

\subsubsection{Step 2: Estimation of \texorpdfstring{$\theta_\D$}{θD}}
\label{method:step2}

We estimate $\theta_\D$ by optimizing the expectation of the objective with regard to $p(\theta_\T)$, some prior distribution of $\theta_\T$.
\begin{equation}\label{eq:adaptive_train}
    \theta_\D^* = \arg \min_{\theta_\D} \mathbb{E}_{p(\theta_\T)} \! \left[ \mathcal{L}_S(\theta_\T, \theta_\D) \!+\! \lambda \mathcal{R}_{S_X}(\theta_\T, \theta_\D) \right],
\end{equation}
where $\mathcal{L}_S$ and $\mathcal{R}_{S_X}$ are computed with regard to the deep grey-box model $\mathcal{C}(f_\T, f_\D(f_\T); x)$ with the adaptive data-driven model, $f_\D(f_\T)$.
We estimate the objective in \cref{eq:adaptive_train} with $m$ samples of $\theta_\T$ drawn from $p(\theta_\T)$.
We empirically found that $m=1$ was sufficient in our experiment.

The $\lambda \mathcal{R}_{S_X}$ term in \cref{eq:adaptive_train} is present just for generality.
When one has multiple candidates of $\mathcal{R}$, which is one of our motivating situations, it is unclear which $\mathcal{R}$ we should use in \cref{eq:adaptive_train}.
We suggest two options:
One is to set $\mathcal{R}$ as the sum of all the candidates and use a small value of $\lambda$.
This is reasonable because it prevents the candidate $\mathcal{R}$s from getting overly large but still does not strongly impose any one of them.
Another option is to set $\lambda=0$, with which we no longer need to set any $\mathcal{R}$.
This is also a reasonable choice because the minimization of $\mathcal{R}$, if it matters, can happen later in the post-estimation phase as we will see.

We suppose that we have some $p(\theta_\T)$ from which we can draw $\theta_\T$s.
It is typically available from domain knowledge concerning the theory-based model.
For example, it can be the uniform distribution on a plausible range of the parameters.
It is also technically possible to let $p(\theta_\T)$ have unknown parameters and estimate them from data with the reparametrization trick.

\subsubsection{Step 3: Post-Estimation Options}
\label{method:step3}

The optimization in \cref{eq:adaptive_train} leaves $\theta_\T$ undetermined, while the model works with any different values of $\theta_\T$ within a region having a reasonable mass of $p(\theta_\T)$.
It allows us to take the following options in the post-estimation phase.

\paragraph{Explore the Landscape of \texorpdfstring{$\mathcal{R}$}{R}}

We can compute the values of $\mathcal{R}$ for different $\theta_\T$s, only with the cost of the forward evaluation of the functions.
If $\theta_\T$ is up to two-dimensional, we can directly draw the landscape of $\mathcal{R}$.
If $\dim\theta_\T>2$ but remains moderate (say $\lessapprox 20$), we can watch the variations of $\mathcal{R}$ by varying each pair of $\theta_\T$'s elements while fixing the others at some reference values.
The same discussion applies to the analysis of the supervised loss, $\mathcal{L}$.

How should we utilize the landscape of $\mathcal{R}$?
Although it is just up to a user's policy and belief in each application, as a general practice, we suggest utilizing it for assessing the reliability of potential estimation based on the $\mathcal{R}$.
For example, if an $\mathcal{R}$ does not have clear extrema along some axes of $\theta_\T$, it implies that those elements of $\theta_\T$ are not quite identifiable under such an $\mathcal{R}$.
It can contribute to the estimation's trustworthiness, e.g., by withholding interpretation about some parameters.
Another general usage is to compare different candidates of $\mathcal{R}$.
The landscapes can give an intuition about the nature of each $\mathcal{R}$, which is useful for a user to choose one (or more) out of the candidates; we will see concrete use cases later in numerical experiments.

Models with high-dimensional $\theta_\T$ basically remain an open challenge; analyzing and visualizing a high-dimensional parameter space are very challenging in general.
One of the options is to get an overview of the landscape via dimensionality reduction of the parameter space by techniques such as random projection and principal component analysis.
With that being said, even if $\theta_\T$ is high-dimensional, we can still benefit from the proposed formulation for the point estimation or the posterior inference discussed below.

\paragraph{Point Estimation of \texorpdfstring{$\theta_\T$}{θT} (optional)}

After analyzing $\mathcal{R}$s and choosing one to use, a user may want to select $\theta_\T$ as
\begin{equation}\label{eq:adaptive_test_optim}
    \theta_\T^* = \arg \min_{\theta_\T} \mathcal{R}_{S'_X}(\theta_\T, \theta_\D^*).
\end{equation}
Depending on the application, $S'_X$ may be a singleton of a test sample, a set of test samples, the union of training and test sets, or a set of grid points on $\mathcal{X}$.
After choosing a specific $\theta_\T$, one does not need to re-run the optimization in \cref{eq:adaptive_train} because the estimated $f_\D$ works adaptively to $\theta_\T$.
With that being said, it is also an option to drop $f_\D$'s dependency on $\theta_\T$ and run the optimization in \cref{eq:inductive} or \eqref{eq:transductive} only with regard to $\theta_\D$.

Although the end result of \cref{eq:adaptive_test_optim} (i.e., a point estimation) has the same form with that of existing frameworks in \cref{eq:inductive} or \eqref{eq:transductive}, our framework can benefit from the decoupled nature of the optimization, that is, $\theta_\D$ and $\theta_\T$ are optimized individually in \cref{eq:adaptive_train} and \cref{eq:adaptive_test_optim}, respectively.
In \cref{eq:adaptive_train}, $\mathcal{R}$ may be used solely for regularizing the behavior of $f_\D$.
In contrast, in \cref{eq:adaptive_test_optim}, $\mathcal{R}$ is minimized for selecting $\theta_\T$.
Such a clear distinction of the semantics of $\mathcal{R}$ in each scene can make the tuning of $\lambda$ easier.
Moreover, the decoupled optimization allows us to use different optimizers and datasets for estimating $\theta_\D$ and $\theta_\T$.

When we should estimate different $\theta_\T$ for each query $x$, we have two options:
One is to solve \cref{eq:adaptive_test_optim} with $S'_X$ being the singleton, $S'_X=\{x\}$, which is inefficient when there are a large number of queries.
Another option is to solve
\begin{equation}\label{eq:adaptive_test_optim_encoder}
    \theta_h^* = \arg \min_{\theta_h} \sum_x \mathcal{R}_{S'_X}(\theta_\T=h(x), \theta_\D^*),
\end{equation}
where $h\colon \mathcal{X} \to \Theta_\T$ is a trainable model parameterized with $\theta_h$ and is used for inferring $\theta_\T$ given $x$ in an amortized manner (i.e., so-called an encoder).

\paragraph{Distribution Estimation of \texorpdfstring{$\theta_\T$}{θT} (optional)}

We can also consider a distribution estimation of $\theta_\T$.
Recall that in deep grey-box models, $\mathcal{L}$ (i.e., the supervision from labeled data) is no longer informative to decide the value of $\theta_\T$ due to the expressive power of $f_\D$, and only $\mathcal{R}$ dictates $\theta_\T$.
Hence, it is reasonable to define a distribution of $\theta_\T$ only with $\mathcal{R}$ as
\begin{equation}
    p(\theta_\T \mid \text{data}) \propto p(\theta_\T) \exp \{ - \beta \mathcal{R}_{S'_X}(\theta_\T, \theta_\D^*) \},
\end{equation}
for some $\beta>0$.
If we compute the full landscape of $\mathcal{R}$ for a low-dimensional $\theta_\T$, it is equivalent to having this distribution explicitly.
Even though $p(\theta_\T \mid \text{data})$ cannot be normalized when $\theta_\T$ is high dimensional, we can efficiently draw samples from the distribution using MCMC.

\subsection{Discussion}

The increased model complexity of $f_\D$ does not come with an additional need for real training data because we can draw as many random samples of $\theta_\T$ as the computational resources allow in the estimation process.
As a result, while the learning problem becomes (hopefully slightly) more complicated than the original ones, it would not make the problem significantly more challenging.
We will empirically confirm it through numerical experiments.

Though the full applicability of the proposed framework is limited to moderate-dimensional $\theta_\T$, in practice, it would not significantly limit applications because theory-based models often have (or should have) a small number of parameters.
Models with high-dimensional $\theta_\T$ such as unknown fields are a challenging open problem.

As repeatedly argued, the main characteristic of the proposed formulation is that it makes it easy to analyze $\mathcal{R}$s empirically.
Here, care should be taken not to reuse the same data both for such an analysis and the estimation with $\mathcal{R}$ selected via the former analysis.
We actually commit ``data reuse'' in parts of the experiments in \cref{exp}, which is admittable only because our purpose is not to analyze the data but to compare the different estimation schemes.

\subsection{A Numerical Example}
\label{method:example}

We show a scenario of how the framework could be used.
We generate data from $y = \sin(x) + \cos(x) + e$, where $x \in \mathcal{X}=[-\pi, \pi]$, and $e \sim \mathcal{N}(0, 0.1^2)$.
We use a deep grey-box model $\mathcal{C}(f_\T, f_\D; x) = f_\T(x) + f_\D(f_\T,x)$ with $f_\T(x) = a \sin(x + c)$ and $f_\D(f_\T,x) = \operatorname{MLP}(x, \theta_\T, f_\T(x))$, where $\theta_\T=[a, c]$ is the $f_\T$'s parameters.
$\operatorname{MLP}(\cdot)$ is a feed-forward neural net.
% with two hidden layers of size $16$ and the leaky ReLU activation function between the layers.
We feed $\operatorname{MLP}$ with $[x; \theta_\T; f_\T(x)] \in \mathbb{R}^4$, and it returns values in $\mathbb{R}$.
Note that our $f_\D$ is aware of $\theta_\T$ and can work adaptively to different values of $\theta_\T$, as advised in Step~1 (\cref{method:step1}).

The next step, Step~2 (\cref{method:step2}), is to optimize the expected objective in \cref{eq:adaptive_train}.
We run it with the squared error, $\mathcal{L} = \sum (y - \mathcal{C}(f_\T, f_\D; x))^2$, and $\lambda=0$ (i.e., we do not specify $\mathcal{R}$).
We let $p(\theta_\T)$ be the uniform distribution over $[0,2] \times [-\pi,\pi]$.
As a result, we obtain $\theta_\D^*$.

As Step~3 (\cref{method:step3}), we overview the landscapes of some candidates of $\mathcal{R}$.
In \cref{fig:toy1_Rmap}, we visualize the values of the following $\mathcal{R}$s over a grid of $\theta_\T=[a,c]$:
\begin{itemize}[itemsep=0pt,topsep=0pt]
    \item $\mathcal{R}_\text{normD} \!=\! \sum_x f_\D(x)^2$ (i.e., $f_\D$ should work minimally);
    \item $\mathcal{R}_\text{corr} = \left\vert \sum_x f_\T(x) \cdot f_\D(x) \right\vert$ (i.e., $f_\T$ and $f_\D$ should work uncorrelatedly); and
    \item $\mathcal{R}_\text{normdif} = \left\vert \sum_x f_\T(x)^2 - \sum_x f_\D(x)^2 \right\vert$ (i.e., $f_\T$ and $f_\D$ should work to the same extent).
\end{itemize}
The leftmost heatmap shows $\mathcal{R}_\text{normD}$; we see that it reaches the minimum around $[a,c]=[\sqrt{2},\pi/4]$, which is natural because $\sin(x) + \cos(x) = \sqrt{2} \sin (x + \pi/4)$.
It serves as a ``truth'' for those who believe that $f_\D$ should work minimally.
However, the case is not over if one has prior knowledge dictating that $f_\T$ and $f_\D$ should be uncorrelated.
Such a user would like to use $\mathcal{R}_\text{corr}$, which is visualized in the second heatmap from the left.
We see that multiple $\theta_T$s are practically equivalent under this regularizer.
If we want to further narrow down the possible choices of $\theta_\T$, we need to consider more criteria.
Let us see $\mathcal{R}_\text{normdif}$, which is shown in the center heatmap.
We find that it is not enough, either, to select a (few) $\theta_\T$ value(s).
Summing $\mathcal{R}_\text{corr}$ and $\mathcal{R}_\text{normdif}$ results in the landscape shown in the fourth heatmap from the left, where we see two minima.
In practice, this would be the most reasonable regularizer that could be designed and explained through these empirical analyses of $\mathcal{R}$s.

\input{fig_toy1_Rmap}

We provide one more visualization in the rightmost panel of \cref{fig:toy1_Rmap}, just for completeness.
It shows the values of $\mathcal{R}_\text{corr}+\mathcal{R}_\text{normdif}+c^2$, which has a single extremum.
However, the $c^2$ term can hardly be specified in practice without the knowledge of the data-generating process.

A message of the above story is that we should analyze the landscape of $\mathcal{R}$, instead of uncritically optimizing it.
Suppose that we follow the prior knowledge dictating that $f_\T$ and $f_\D$ should be uncorrelated.
Applying a standard learning scheme (i.e., optimize $\mathcal{L} + \lambda \mathcal{R}_\text{corr}$ without the adaptivity of $f_\D$) would result in $\theta_\T$ being practically a random choice from one of the many local minima of $\mathcal{R}_\text{corr}$ that we see in \cref{fig:toy1_Rmap} (2nd from the left).
This is hardly meaningful since these parameters are supposed to have an interpretation within the context of the relevant domain knowledge.

\input{fig_toy1_performance}

Let us take one of the further options of Step~3, the point estimation.
Suppose using $\mathcal{R} = \mathcal{R}_\text{corr} + \mathcal{R}_\text{normdif} + c^2$.
We compare our framework, \cref{eq:adaptive_train,eq:adaptive_test_optim}, with inductive learning, \cref{eq:inductive}, and transductive learning, \cref{eq:transductive}.
We vary the value of $\lambda$ from $0.001$ to $0.1$.
\Cref{fig:toy1_performance} reports the squared error, $\mathcal{L}$, and the regularizer, $\mathcal{R}$, computed on a test set for each configuration.
Our framework achieves small values of $\mathcal{R}$ practically for any value of $\lambda$; this is not the case for inductive and transductive learning.
The slightly large $\mathcal{L}$ for the proposed framework with $\lambda=0.1$ is probably because the training of the adaptive $f_\D$ was inhibited by the large value of $\lambda$.
This result supports the use of a small value of $\lambda$ in the optimization of \cref{eq:adaptive_train}.
We also confirmed that our framework worked well with $\lambda=0$.

%% file: fig_toy1_Rmap.tex
\begin{figure*}[t]
    \centering
    \setlength{\tabcolsep}{1pt}
    {\fontsize{9pt}{10pt}\selectfont\begin{tabular}{ccccc}
        \hspace{8pt}$\mathcal{R}_\text{normD}$ &
        \hspace{8pt}$\mathcal{R}_\text{corr}$ &
        \hspace{8pt}$\mathcal{R}_\text{normdif}$ &
        \hspace{8pt}$\mathcal{R}_\text{corr} + \mathcal{R}_\text{normdif}$ &
        \hspace{8pt}$\mathcal{R}_\text{corr} \!+\! \mathcal{R}_\text{normdif} \!+\! c^2$
        \\
        \includegraphics[clip,height=26mm]{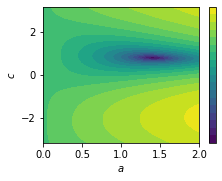} &
        \includegraphics[clip,height=26mm]{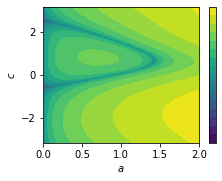} &
        \includegraphics[clip,height=26mm]{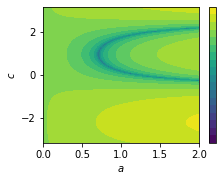} &
        \includegraphics[clip,height=26mm]{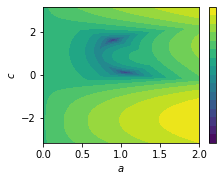} &
        \includegraphics[clip,height=26mm]{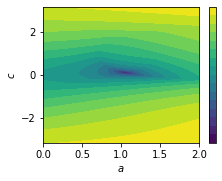}
    \end{tabular}}
    \caption{Landscapes of some regularizers and their combinations, for $f_\D$ trained by \cref{eq:adaptive_train} with $\lambda=0$ on the toy dataset in \cref{method:example}. The horizontal and vertical axes correspond to $a$ and $c$ of $f_\T$, respectively.}
    \label{fig:toy1_Rmap}
\end{figure*}

%% file: fig_toy1_performance.tex
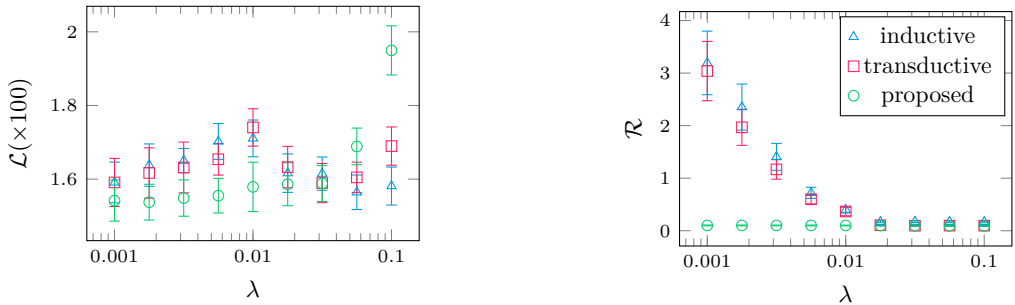
\begin{figure}
    \centering
    \begin{minipage}[t]{0.48\linewidth}
        \vspace{0pt}\centering
        \begin{tikzpicture}
            \begin{axis}[
                xmode = log,
                log ticks with fixed point,
                width = 6.0cm,
                height = 4.7cm,
                xlabel = {$\lambda$},
                ylabel = {$\mathcal{L} (\times 100)$},
            ]
              \addplot+[myblue, mark=triangle, only marks] plot[error bars/.cd, y dir=both, y dir=both, y explicit] table[x=coeff, y=Lmeanx100, y error=Lstderrx100] {plots_new/stats_toy1_inductive.txt};
              \addplot+[myred, mark=square, only marks] plot[error bars/.cd, y dir=both, y dir=both, y explicit] table[x=coeff, y=Lmeanx100, y error=Lstderrx100] {plots_new/stats_toy1_transductive.txt};
              \addplot+[mygreen, mark=o, only marks] plot[error bars/.cd, y dir=both, y dir=both, y explicit] table[x=coeff, y=Lmeanx100, y error=Lstderrx100] {plots_new/stats_toy1_adaptive.txt};
            \end{axis}
        \end{tikzpicture}
    \end{minipage}
    \begin{minipage}[t]{0.48\linewidth}
        \vspace{0pt}\centering
        \begin{tikzpicture}
            \begin{axis}[
                xmode = log,
                log ticks with fixed point,
                width = 6.0cm,
                height = 4.7cm,
                xlabel = {$\lambda$},
                ylabel = {$\mathcal{R}$},
                legend style={nodes={scale=0.9, transform shape}},
            ]
              \addplot+[myblue, mark=triangle, only marks] plot[error bars/.cd, y dir=both, y dir=both, y explicit] table[x=coeff, y=Rmean, y error=Rstderr] {plots_new/stats_toy1_inductive.txt};
              \addplot+[myred, mark=square, only marks] plot[error bars/.cd, y dir=both, y dir=both, y explicit] table[x=coeff, y=Rmean, y error=Rstderr] {plots_new/stats_toy1_transductive.txt};
              \addplot+[mygreen, mark=o, only marks] plot[error bars/.cd, y dir=both, y dir=both, y explicit] table[x=coeff, y=Rmean, y error=Rstderr] {plots_new/stats_toy1_adaptive.txt};
              \legend{inductive, transductive, proposed}
            \end{axis}
        \end{tikzpicture}
    \end{minipage}
    \caption{Main loss $\mathcal{L}$ and regularizer $\mathcal{R}$ on test data for the models learned with different values of regularization hyperparameter, $\lambda$, on the toy dataset in \cref{method:example}. Error bars show the standard errors by 20 random trials.}
    \label{fig:toy1_performance}
\end{figure}

%% file: 04_relatedwork.tex
\section{RELATED WORK}
\label{related}

\paragraph{Deep Grey-box Modeling}

Combination of deep neural nets and theory-driven models has been studied in various contexts \citep[e.g., recent studies include][]{yinAugmentingPhysicalModels2021,takeishiPhysicsintegratedVariationalAutoencoders2021,wehenkelRobustHybridLearning2022,reichsteinDeepLearningProcess2019,azariIncorporatingPhysicalKnowledge2020,bikmukhametovCombiningMachineLearning2020,eIntegratingMachineLearning2020,willardIntegratingPhysicsbasedModeling2020,karniadakisPhysicsinformedMachineLearning2021,vonruedenInformedMachineLearning2021,wangPhysicsguidedDeepLearning2021,konUnifyingModelbasedNeural2022}; we cannot enumerate more from a large number of studies in older times and in application domains, so readers are recommended consulting the references in the aforementioned papers.

Among most related studies, \citet{yinAugmentingPhysicalModels2021} formulate the learning of deep grey-box models as a constrained optimization problem where they minimize $\mathcal{R}$ under some constraint on $\mathcal{L}$ (e.g., $\mathcal{L}=0$).
\citet{takeishiPhysicsintegratedVariationalAutoencoders2021} study a family of deep grey-box models in the context of variational autoencoders.
In both of these studies, the theory- and data-driven models are learned once and together.
\citet{wehenkelRobustHybridLearning2022} proposes a method to deal with distribution shift utilizing deep grey-box models.
Their method is applicable to our framework, too.

\paragraph{Other Relevant Contexts}

When combining the models, a theory-driven model may constitute the final layer (in contrast, we are interested in the opposite, i.e., the data-driven model at the final layer).
\citet{arikInterpretableSequenceLearning2020} developed an epidemiological model whose parameters are predicted by neural nets.
Physics-informed neural nets \citep{raissiPhysicsinformedNeuralNetworks2019} and their variants belong also to this category because their loss, the residual of a differential equation, can be regarded as the final layer.
% More generally, simulation-based inference with machine learning \TODO{cite} and reinforce/imitation learning with simulators \TODO{cite} also fall in this category as the loss/reward is computed based on physics simulators.
\citet{schnellHalfinverseGradientsPhysical2022} propose an optimization method for neural nets when the loss is computed based on some physics-based models.

\paragraph{Learning Schemes}

Our method can minimize a part of the objective at prediction time.
This is also the case with model-agnostic meta-learning (MAML) \citep{finnModelagnosticMetalearningFast2017}.
Differently from our setting, MAML is meant for settings where one has many tasks and wants to adapt to a new task at prediction time.
Meta-tailoring \citep{aletTailoringEncodingInductive2021} can be thought of as a variant of MAML with each sample being one task.
Nonetheless, these methods do not necessarily allow us to efficiently explore objectives' landscape.

\citet{seoControllingNeuralNetworks2021} proposed a method to learn deep neural nets with supervision from both data and rules.
They suggest ``marginalizing out'' a parameter during training so that we can choose it freely in the inference time.
Despite such similarities, there are several differences from our framework, e.g., in the target of the marginalization and in how theory is incorporated, which prevent direct comparison.

% domain randomization ... change of physics is supposed to cause almost random change of outcome, which is not the case for us ...

%% file: 05_experiment.tex
\section{EXPERIMENTS}
\label{exp}

We demonstrate the capability of the proposed framework to explore the landscape of regularizers.
We also compare the proposed framework with existing learning schemes when we take the option of point estimation.
The source codes of the experiments are available at \url{https://github.com/n-takeishi/deepgreybox}.

\subsection{Datasets, Models, and Optimization}

In all cases, $\mathcal{L}$ is the mean squared error.

\paragraph{Controlled Pendulum}

We use time-series data of a frictionless compound pendulum controlled by an unknown regulator.
The dataset comprises pairs $(x=s_t, y=(s_{t+1}, \dots, s_{t+10}))$, where $s_t = [s_{t,1}, s_{t,2}]$ is the state (angle $s_{t,1}$ and angular velocity $s_{t,2}$) of the pendulum.
Hence, the task is to predict the next 10 steps future given a state.

We use the following model: $\mathcal{C}(f_\T, f_\D; x) = \operatorname{ODESolve} [$ $\dot{s}_t = f_\T(s_t) + f_\D(s_t, \theta_T, f_\T(s_t)) \mid s_0 = x ]$, where $f_\T(s_t) = [ s_{t,2}, \frac{3\theta_\T}2 \sin(s_{t,1})]$, and $f_\D(s_t, \theta_\T, f_\T(s_t))$ is a network with fully-connected layers.
$f_\D$ should mimic the behavior of the unknown controller.
It is not obvious what $\mathcal{R}$ we should use because we do not know the controller's nature.
In \cref{eq:adaptive_train}, we set $\mathcal{R}=\mathcal{R}_\text{normD} + \mathcal{R}_\text{corr}$, where $\mathcal{R}_\text{normD}=\sum_x \Vert f_\D(x) \Vert_2^2$ and $\mathcal{R}_\text{corr} = \vert \sum_x f_\T(x) \cdot f_\D(x) \vert$ are candidate $\mathcal{R}$s, with a small coefficient $\lambda=0.001$.

\paragraph{Reaction--Diffusion System}

We generated data from the two-dimensional two-component reaction--diffusion system of the FitzHugh--Nagumo type: $\partial u / \partial t = 0.0015 \Delta u + u - u^3 - v + 0.005$ and $\partial v / \partial t = 0.005 \Delta v + u - v$, where $u,v \in \mathbb{R}^{32 \times 32}$ are the concentration of two substances.
The dataset comprises pairs $(x=s_0, y=(s_{1},s_{2},\dots,s_{15}))$ where $s_t=[u_t, v_t]$.
We note that similar configurations of the same system have been used in previous related studies \citep{yinAugmentingPhysicalModels2021,wehenkelRobustHybridLearning2022}.

We use a grey-box neural ODE as in the previous example, with $f_\T(s_t) = [a \Delta u_t, b \Delta v_t]$ whose parameter is $\theta_\T = [a,b]$ and $f_\D(s_t, \theta_\T) = \operatorname{ConvNet}_{\theta_\T}(s_t)$, a convolutional net whose filter is partially parameterized by $\theta_\T$.
In the previous studies \citep{yinAugmentingPhysicalModels2021,wehenkelRobustHybridLearning2022}, it has been reported that $\mathcal{R}_\text{normD}=\sum_x \Vert f_\D(x) \Vert_2^2$ works to some extent, yet without detailed analysis of its behavior.
Hence, we use $\mathcal{R}=\mathcal{R}_\text{normD} + \mathcal{R}_\text{corr}$ with a small value of $\lambda$ to examine the both regularizers as candidates.

\paragraph{Predator--Prey System}

We use real data of a planktonic predator--prey system \citep{blasiusLongtermCyclicPersistence2020}.
From the original data, we extracted the measurements of the population densities of the prey (unicellular algae) and the predator (rotifer).
We split the original time-series into subsequences of length $11$ [days] and created a dataset comprising $x=y=(s_t, \dots, s_{t+10})$ where $s_t = [s_{t,1}, s_{t,2}] \in \mathbb{R}^2$ is the state (prey density $s_{t,1}$ and predator density $s_{t,2}$) at day $t$.
Hence, the task is autoencoding.

The model comprises a decoder and two encoders.
The decoder is again a grey-box neural ODE: $\mathcal{C}(f_\T, f_\D; x) = \operatorname{ODESolve} \left[ \dot{s}_t = f_\T(s_t) + f_\D(z,\! s_t,\! \theta_T,\! f_\T(s_t)) \mid s_0 = x_0 \right]$, where $z$ is a latent variable.
$f_\T$ is the Lotka--Volterra equations: $f_\T(s_t) = [\alpha s_{t,1} - \beta s_{t,1} s_{t,2}, -\gamma s_{t,2} + \delta s_{t,1} s_{t,2}]$, where $\theta_\T = [\alpha, \beta, \gamma, \delta]$ ($\alpha$: prey's growth rate,  $\beta$: prey's decay rate, $\gamma$: predator's decay rate, and $\delta$: predator's growth rate).
One of the encoders is a neural net with fully-connected layers that takes $x$ as input and outputs $z$.
We run \cref{eq:adaptive_train} with $\mathcal{R}=\sum_x \Vert f_\D(x) \Vert_2^2$ and $\lambda=0.001$ to explore the $\mathcal{R}$'s landscape.
Then, we run \cref{eq:adaptive_test_optim_encoder} to get another encoder that infers $\theta_\T$ for each $x$.

% ------------------------------------------------------------------

\subsection{Explore the Landscape of \texorpdfstring{$\mathcal{R}$}{R}}

\input{fig_pendulum_Rmap}

\Cref{fig:pendulum_Rmap} shows, for the \textbf{controlled pendulum} dataset, the values of $\mathcal{R}_\text{normD}$ and $\mathcal{R}_\text{corr}$ for different $\theta_\T$s.
We can see that the two regularizers have slightly different peaks.
While it is not easy to decide which one is better in some sense, it is notable that our method can provide such empirical clues on the differences between the regularizers.
% \footnote{The data-generating value is $\theta_\T=10$, so in this sense $\mathcal{R}_\text{corr}$ is ``better.'' But, this is something we never know in practice.}
The rightmost plot of \cref{fig:pendulum_Rmap} shows the normalized root mean squared errors (NRMSEs) of the prediction on the test set.
The model can predict well regardless of the value of $\theta_\T$.

\input{fig_reaction-diffusion_Rmap}

\Cref{fig:reaction-diffusion_Rmap} visualizes the $\mathcal{R}$ candidates and the test RMSE for the \textbf{reaction--diffusion system} dataset, analogously to what was explained in the previous paragraph.
$\mathcal{R}_\text{normD}$ has a peak at some value, while the extrema of $\mathcal{R}_\text{corr}$ are much less clear.
We found that the product of these two terms, $\mathcal{R}_\text{normD} \cdot \mathcal{R}_\text{corr}$ , also have clear, if not unique, peaks.
Interestingly, one of the local minima of $\mathcal{R}_\text{normD} \cdot \mathcal{R}_\text{corr}$ nicely points the data-generating value of $\theta_\T=[a,b]$, that is, $[a,b]=[0.0015,0.005]$.
On the other hand, $\mathcal{R}_\text{normD}$, which has been used in the previous studies, does not point to this value as the minimum.
This suggests that, again, we should empirically analyze $\mathcal{R}$ to know how different $\mathcal{R}$ may result in different $\theta_\T$ estimations, instead of optimizing it uncritically.
We are \emph{not} suggesting $\mathcal{R}_\text{normD} \cdot \mathcal{R}_\text{corr}$ is a ``better'' regularizer; we never know the truth of $\theta_\T$ in practice and thus cannot evaluate anything from such a viewpoint.

\input{figs_predator-prey}

\Cref{fig:predator-prey} (left) shows an outcome of the estimated model on a test sequence of the \textbf{predator--prey system} dataset.
We took a sliding window of length $11$ from the long test sequence, applied the model to the subsequences from the window, and calculated the averages of the outputs for the overlapping steps.
The value of $\Vert f_\D(x) \Vert_2^2$ at each time is informative to assess how well the $f_\T$ could explain the data.
We can observe that the latter half of the sequence is relatively well explained by $f_\T$.
We can also analyze the landscape of the regularizer, $\mathcal{R}=\sum_x \Vert f_\D(x) \Vert_2^2$.
\Cref{fig:predator-prey} (right) shows the slices of the values of the $\mathcal{R}$.
Each slice is computed by fixing the values of the remaining elements of $\theta_\T$ at the encoder's outputs.
An observation is that while $\mathcal{R}$ seems to have some extrema along the directions of $\beta$, $\gamma$, $\delta$, it looks quite flat along the direction of $\alpha$.
It implies the possible difficulty of inferring $\alpha$, prey's growth rate without interaction, from the data, and thus we should not immediately interpret at least the results of $\alpha$'s inference.

% ------------------------------------------------------------------

\subsection{Point Estimation of \texorpdfstring{$\theta_\T$}{θT}}

\input{fig_pendulum_performance}

\input{fig_reaction-diffusion_performance}

We performed one of the further options in the proposed framework.
For the \textbf{controlled pendulum} dataset, we set the regularizer as $\mathcal{R}=\mathcal{R}_\text{corr}$ and performed the point estimation by \cref{eq:inductive} (inductive), \cref{eq:transductive} (transductive), and \cref{eq:adaptive_test_optim} (proposed).
\Cref{fig:pendulum_performance} shows the values of $\mathcal{L}$ and $\mathcal{R}$ on the test set.
$\mathcal{L}$ is not well minimized with a large $\lambda$ for all the methods, which is natural.
A notable difference is that by the proposed method, $\mathcal{R}$ is well minimized even with small $\lambda$, though it is not surprising because the optimization is decoupled.
In contrast, the inductive and transductive learning schemes result in larger $\mathcal{R}$ values with small $\lambda$s, due to which striking a good value of $\lambda$ may be difficult for these learning schemes.
\Cref{fig:reaction-diffusion_performance} reports basically the same thing but for the \textbf{reaction--diffusion system} dataset.

%% file: fig_pendulum_Rmap.tex
\begin{figure}[t]
    \begin{minipage}[t]{0.32\linewidth}
        \vspace{0pt}\centering
        \begin{tikzpicture}
            \begin{axis}[
                width = 4.5cm,
                height = 4.0cm,
                xlabel = {$\theta_\T$},
                ylabel = {$\mathcal{R}_\text{normD}$},
            ]
                \addplot[mybrown, very thick] table[x=thT, y=normD] {plots_new/pendulum_Rmap.txt};
                %\addplot[mygray, dashed, ultra thick] coordinates {(10,2.4)(10,2.7)};
            \end{axis}
        \end{tikzpicture}
    \end{minipage}
    \hfill
    \begin{minipage}[t]{0.32\linewidth}
        \vspace{0pt}\centering
        \begin{tikzpicture}
            \begin{axis}[
                width = 4.5cm,
                height = 4.0cm,
                xlabel = {$\theta_\T$},
                ylabel = {$\mathcal{R}_\text{corr}$},
            ]
                \addplot[mybrown, very thick] table[x=thT, y=abs_dotTD] {plots_new/pendulum_Rmap.txt};
                %\addplot[mygray, dashed, ultra thick] coordinates {(10,-0.3)(10,5.0)};
            \end{axis}
        \end{tikzpicture}
    \end{minipage}
    \hfill
    \begin{minipage}[t]{0.32\linewidth}
        \vspace{0pt}\centering
        \begin{tikzpicture}
            \begin{axis}[
                width = 4.5cm,
                height = 4.0cm,
                xlabel = {$\theta_\T$},
                ylabel = {NRMSE [\%]},
                ymin=0, ymax=0.5,
            ]
                \addplot[mybrown, very thick] table[x=thT, y=nrmse_prc] {plots_new/pendulum_Rmap.txt};
            \end{axis}
        \end{tikzpicture}
    \end{minipage}
    \caption{(\emph{Left} \& \emph{center}) Landscapes of the regularization terms, $\mathcal{R}_\text{normD}$ and $\mathcal{R}_\text{corr}$ computed on the test set of the \textbf{controlled pendulum} dataset. (\emph{Right}) Test NRMSEs.}
    \label{fig:pendulum_Rmap}
\end{figure}
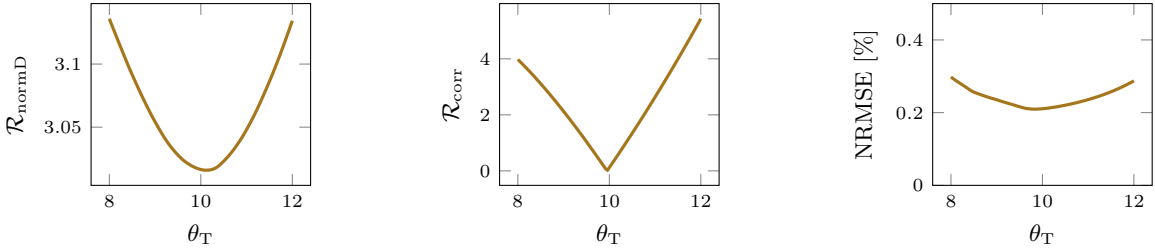

%% file: fig_reaction-diffusion_Rmap.tex
\begin{figure*}[t]
    \centering
    \setlength{\tabcolsep}{1pt}
    {\fontsize{9pt}{10pt}\selectfont\begin{tabular}{ccccc}
        \hspace{8pt}$\mathcal{R}_\text{normD}$ &
        \hspace{8pt}$\mathcal{R}_\text{corr}$ &
        \hspace{8pt}$\mathcal{R}_\text{normD} + \mathcal{R}_\text{corr}$ &
        \hspace{8pt}$\mathcal{R}_\text{normD} \cdot \mathcal{R}_\text{corr}$ &
        \hspace{0pt}RMSE
        \\
        \includegraphics[clip,height=26mm]{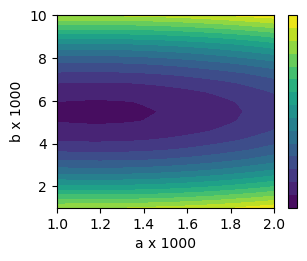} &
        \includegraphics[clip,height=26mm]{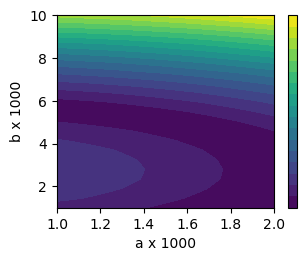} &
        \includegraphics[clip,height=26mm]{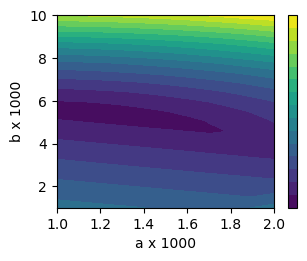} &
        \includegraphics[clip,height=26mm]{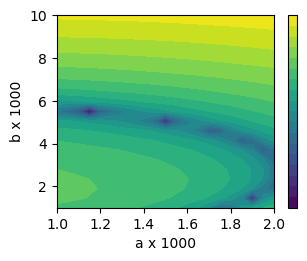} &
        \includegraphics[clip,height=26mm]{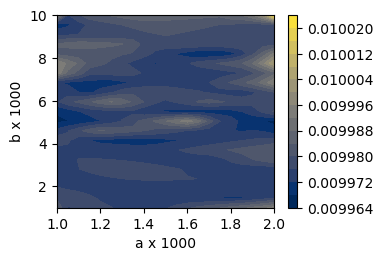}
    \end{tabular}}
    \caption{(\emph{Left four}) Landscapes of $\mathcal{R}_\text{normD}$, $\mathcal{R}_\text{corr}$, and their combinations, for the \textbf{reaction--diffusion system} dataset. The horizontal and vertical axes correspond to $a$ and $b$ of $f_\T$, respectively. (\emph{Right}) Test RMSEs.}
    \label{fig:reaction-diffusion_Rmap}
\end{figure*}

%% file: figs_predator-prey.tex
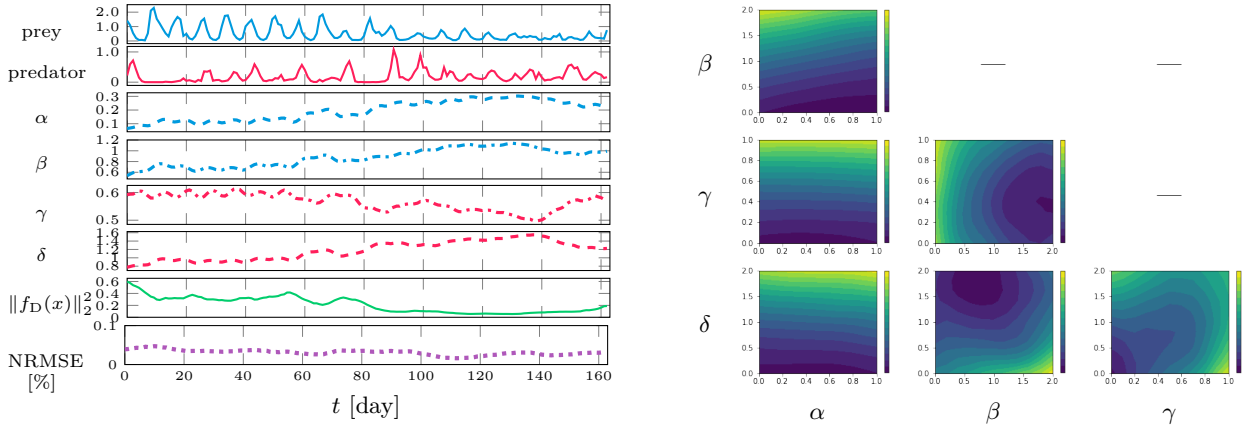
\begin{figure*}[t]
    \centering
    \begin{minipage}[t]{0.51\textwidth}
        \vspace{0pt}\centering
        \setlength{\tabcolsep}{2pt}
        {\scriptsize\begin{tabular}{M{9mm}M{70mm}}
            prey &
            \hspace*{-0pt}\begin{tikzpicture}
                \begin{axis}[
                    width = 8cm,
                    height = 2.1cm,
                    xlabel = {},
                    ylabel = {},
                    xticklabels = {,,},
                    ytick = {0,1,2},
                    yticklabels = {$0$,$1.0$,$2.0$},
                    xmin = 0, xmax = 163,
                    tick label style={font=\tiny},
                ]
                  \addplot[myblue, thick] table[x=t, y=prey] {plots_new/predator-prey_C7.txt};
                \end{axis}
            \end{tikzpicture}
            \\[-9pt]
            predator &
            \hspace*{-0pt}\begin{tikzpicture}
                \begin{axis}[
                    width = 8cm,
                    height = 2.1cm,
                    xlabel = {},
                    ylabel = {},
                    xticklabels = {,,},
                    ytick = {0,1},
                    yticklabels = {$0$,$1.0$},
                    xmin = 0, xmax = 163,
                    tick label style={font=\tiny},
                ]
                  \addplot[myred, thick] table[x=t, y=predator] {plots_new/predator-prey_C7.txt};
                \end{axis}
            \end{tikzpicture}
            \\[-8pt]
            $\alpha$ &
            \hspace*{-0pt}\begin{tikzpicture}
                \begin{axis}[
                    width = 8cm,
                    height = 2.1cm,
                    xlabel = {},
                    ylabel = {},
                    xticklabels = {,,},
                    xmin = 0, xmax = 163,
                    tick label style={font=\tiny},
                ]
                  \addplot[myblue, very thick, dashed] table[x=t, y=b] {plots_new/predator-prey_C7.txt};
                \end{axis}
            \end{tikzpicture}
            \\[-9pt]
            $\beta$ &
            \hspace*{-0pt}\begin{tikzpicture}
                \begin{axis}[
                    width = 8cm,
                    height = 2.1cm,
                    xlabel = {},
                    ylabel = {},
                    xticklabels = {,,},
                    xmin = 0, xmax = 163,
                    tick label style={font=\tiny},
                ]
                  \addplot[myblue, very thick, dash dot] table[x=t, y=p] {plots_new/predator-prey_C7.txt};
                \end{axis}
            \end{tikzpicture}
            \\[-7pt]
            $\gamma$ &
            \hspace*{-0pt}\begin{tikzpicture}
                \begin{axis}[
                    width = 8cm,
                    height = 2.1cm,
                    xlabel = {},
                    ylabel = {},
                    xticklabels = {,,},
                    ytick = {0.5,0.6},
                    yticklabels = {$0.5$,$0.6$},
                    xmin = 0, xmax = 163,
                    tick label style={font=\tiny},
                ]
                  \addplot[myred, very thick, dash dot] table[x=t, y=d] {plots_new/predator-prey_C7.txt};
                \end{axis}
            \end{tikzpicture}
            \\[-9pt]
            $\delta$ &
            \hspace*{-0pt}\begin{tikzpicture}
                \begin{axis}[
                    width = 8cm,
                    height = 2.1cm,
                    xlabel = {},
                    ylabel = {},
                    xticklabels = {,,},
                    xmin = 0, xmax = 163,
                    tick label style={font=\tiny},
                ]
                  \addplot[myred, very thick, dashed] table[x=t, y=r] {plots_new/predator-prey_C7.txt};
                \end{axis}
            \end{tikzpicture}
            \\[-8pt]
            $\Vert f_\D(x) \Vert_2^2$ &
            \hspace*{-0pt}\begin{tikzpicture}
                \begin{axis}[
                    width = 8cm,
                    height = 2.1cm,
                    xlabel = {},
                    ylabel = {},
                    xticklabels = {,,},
                    xmin = 0, xmax = 163,
                    tick label style={font=\tiny},
                ]
                  \addplot[mygreen, thick] table[x=t, y=normD] {plots_new/predator-prey_C7.txt};
                \end{axis}
            \end{tikzpicture}
            \\[-9pt]
            NRMSE [\%] &
            \hspace*{-0pt}\begin{tikzpicture}
                \begin{axis}[
                    width = 8cm,
                    height = 2.1cm,
                    xlabel = {$t$ [day]},
                    ylabel = {},
                    xmin = 0, xmax = 163,
                    ymin = 0.0, ymax = 0.1,
                    ytick = {0,0.1},
                    yticklabels = {$0$,$0.1$},
                    tick label style={font=\tiny},
                ]
                  \addplot[mypurple, ultra thick, dotted] table[x=t, y=nrmse_prc] {plots_new/predator-prey_C7.txt};
                \end{axis}
            \end{tikzpicture}
        \end{tabular}}
        \vspace*{-2ex}
    \end{minipage}
    \hfill
    \begin{minipage}[t]{0.45\textwidth}
        \vspace{-3pt}\centering
        \setlength{\tabcolsep}{2pt}
        \def\arraystretch{1}
        {\small\begin{tabular}{M{5mm}M{22mm}M{22mm}M{22mm}}
            $\beta$ &
            \includegraphics[clip,height=1.8cm,trim=3mm 0 15mm 0]{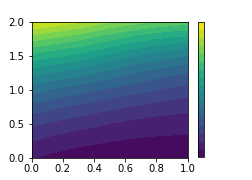} &
            --- &
            ---
            \\[-5pt]
            $\gamma$ &
            \includegraphics[clip,height=1.8cm,trim=3mm 0 15mm 0]{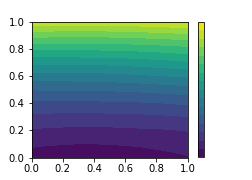} &
            \includegraphics[clip,height=1.8cm,trim=3mm 0 15mm 0]{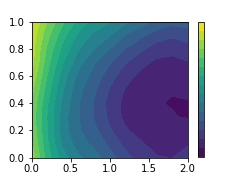} &
            ---
            \\[-5pt]
            $\delta$ &
            \includegraphics[clip,height=1.8cm,trim=3mm 0 15mm 0]{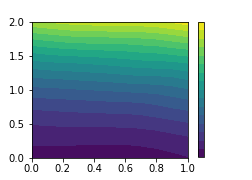} &
            \includegraphics[clip,height=1.8cm,trim=3mm 0 15mm 0]{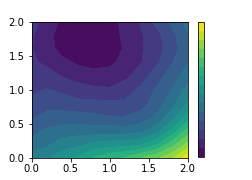} &
            \includegraphics[clip,height=1.8cm,trim=3mm 0 15mm 0]{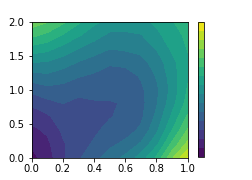}
            \\
            & $\alpha$ & $\beta$ & $\gamma$
        \end{tabular}}
    \end{minipage}
    \caption{(\emph{Left}) A test sequence of the \textbf{predator--prey system} dataset, inferred $\alpha, \beta, \gamma, \delta$, $\Vert f_\D(x) \Vert_2^2$, and prediction NRMSE. (\emph{Right}) Partial landscape of $\mathcal{R}=\sum_x \Vert f_\D(x) \Vert_2^2$ at $t=65$.}
    \label{fig:predator-prey}
\end{figure*}

%% file: fig_pendulum_performance.tex
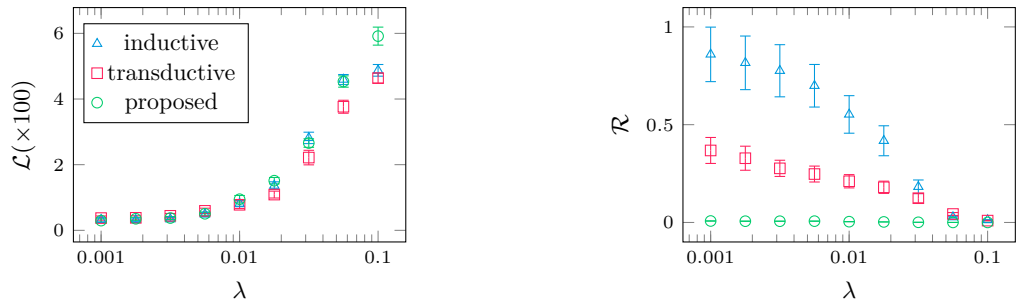
\begin{figure}[t]
    \centering
    \begin{minipage}[t]{0.48\linewidth}
        \vspace{0pt}\centering
        \begin{tikzpicture}
            \begin{axis}[
                xmode = log,
                log ticks with fixed point,
                width = 6.0cm,
                height = 4.7cm,
                xlabel = {$\lambda$},
                ylabel = {$\mathcal{L} (\times 100)$},
                legend style={at={(0.03,0.72)}, anchor=west, nodes={scale=0.9, transform shape}},
            ]
              \addplot+[myblue, mark=triangle, only marks] plot[error bars/.cd, y dir=both, y explicit] table[x=coeff, y=Lmeanx100, y error=Lstderrx100] {plots_new/stats_pendulum_inductive.txt};
              \addplot+[myred, mark=square, only marks] plot[error bars/.cd, y dir=both, y explicit] table[x=coeff, y=Lmeanx100, y error=Lstderrx100] {plots_new/stats_pendulum_transductive.txt};
              \addplot+[mygreen, mark=o, only marks] plot[error bars/.cd, y dir=both, y explicit] table[x=coeff, y=Lmeanx100, y error=Lstderrx100] {plots_new/stats_pendulum_adaptive.txt};
              \legend{inductive, transductive, proposed}
            \end{axis}
        \end{tikzpicture}
    \end{minipage}
    \begin{minipage}[t]{0.48\linewidth}
        \vspace{0pt}\centering
        \begin{tikzpicture}
            \begin{axis}[
                xmode = log,
                log ticks with fixed point,
                width = 6.0cm,
                height = 4.7cm,
                xlabel = {$\lambda$},
                ylabel = {$\mathcal{R}$},
            ]
              \addplot+[myblue, mark=triangle, only marks] plot[error bars/.cd, y dir=both, y explicit] table[x=coeff, y=Rmean, y error=Rstderr] {plots_new/stats_pendulum_inductive.txt};
              \addplot+[myred, mark=square, only marks] plot[error bars/.cd, y dir=both, y explicit] table[x=coeff, y=Rmean, y error=Rstderr] {plots_new/stats_pendulum_transductive.txt};
              \addplot+[mygreen, mark=o, only marks] plot[error bars/.cd, y dir=both, y explicit] table[x=coeff, y=Rmean, y error=Rstderr] {plots_new/stats_pendulum_adaptive.txt};
            \end{axis}
        \end{tikzpicture}
    \end{minipage}
    \caption{Loss $\mathcal{L}$ and regularizer $\mathcal{R}$ on test set of the \textbf{controlled pendulum} dataset. Error bars show the standard errors by 20 random trials.}
    \label{fig:pendulum_performance}
\end{figure}

%% file: fig_reaction-diffusion_performance.tex
\begin{figure}
    \begin{minipage}[t]{0.50\linewidth}
        \vspace{0pt}\centering
        \begin{tikzpicture}
            \begin{axis}[
                xmode = log,
                log ticks with fixed point,
                width = 6.0cm,
                height = 4.7cm,
                xlabel = {$\lambda$},
                ylabel = {$\mathcal{L}$ ($\times 10^4$)},
                xtick = {0.0000001,0.000001,0.00001,0.0001},
                xticklabels = {$10^{-7}$,$10^{-6}$,$10^{-5}$,$10^{-4}$},
                legend style={at={(0.03,0.72)}, anchor=west, nodes={scale=0.9, transform shape}},
            ]
              \addplot+[myblue, mark=triangle, only marks] plot[error bars/.cd, y dir=both, y dir=both, y explicit] table[x=coeff, y=Lmeanx10000, y error=Lstderrx10000] {plots_new/stats_reaction-diffusion_inductive.txt};
              \addplot+[myred, mark=square, only marks] plot[error bars/.cd, y dir=both, y explicit] table[x=coeff, y=Lmeanx10000, y error=Lstderrx10000] {plots_new/stats_reaction-diffusion_transductive.txt};
              \addplot+[mygreen, mark=o, only marks] plot[error bars/.cd, y dir=both, y explicit] table[x=coeff, y=Lmeanx10000, y error=Lstderrx10000] {plots_new/stats_reaction-diffusion_adaptive.txt};
              \legend{inductive, transductive, proposed}
            \end{axis}
        \end{tikzpicture}
    \end{minipage}
    \begin{minipage}[t]{0.46\linewidth}
        \vspace{0pt}\centering
        \begin{tikzpicture}
            \begin{axis}[
                xmode = log,
                log ticks with fixed point,
                width = 6.0cm,
                height = 4.7cm,
                xlabel = {$\lambda$},
                ylabel = {$\mathcal{R}$},
                xtick = {0.0000001,0.000001,0.00001,0.0001},
                xticklabels = {$10^{-7}$,$10^{-6}$,$10^{-5}$,$10^{-4}$}
            ]
              \addplot+[myblue, mark=triangle, only marks] plot[error bars/.cd, y dir=both, y dir=both, y explicit] table[x=coeff, y=Rmean, y error=Rstderr] {plots_new/stats_reaction-diffusion_inductive.txt};
              \addplot+[myred, mark=square, only marks] plot[error bars/.cd, y dir=both, y explicit] table[x=coeff, y=Rmean, y error=Rstderr] {plots_new/stats_reaction-diffusion_transductive.txt};
              \addplot+[mygreen, mark=o, only marks] plot[error bars/.cd, y dir=both, y explicit] table[x=coeff, y=Rmean, y error=Rstderr] {plots_new/stats_reaction-diffusion_adaptive.txt};
            \end{axis}
        \end{tikzpicture}
    \end{minipage}
    \caption{Loss $\mathcal{L}$ and regularizer $\mathcal{R}$ on test set of the \textbf{reaction--diffusion} system dataset. Error bars show the standard errors by 20 random trials.}
    \label{fig:reaction-diffusion_performance}
\end{figure}
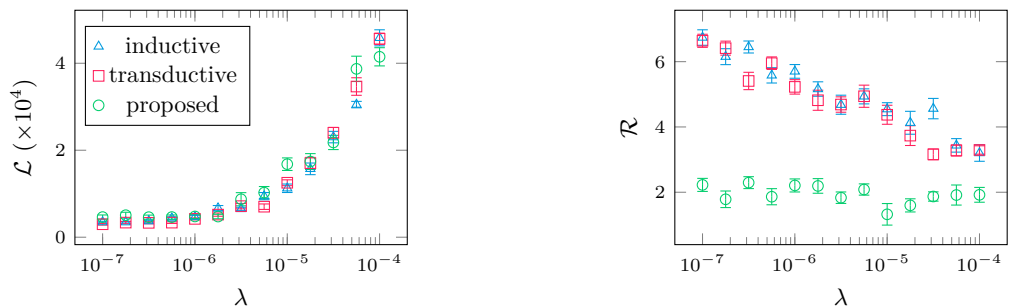

%% file: 06_conclusion.tex
\section{CONCLUSION}

Deep grey-box models are combinations of deep neural networks and theory-driven models.
We argued that, toward trustworthy estimation of the theory-driven part of the model, we should empirically analyze the regularizers we use.
We suggested a slight modification of the network architecture and the optimization objective used in deep grey-box models such that we can empirically analyze regularizers after training.
The proposed formulation is also useful as it decouples the training of the data-driven and theory-driven models.
The main limitation is the applicability to theory-driven models with high-dimensional parameters.
They pose issues in two aspects: computation of the expectation with regard to the theory parameters and visualization of a regularizer's landscape.

%% file: 07_appendix.tex
\appendix

\counterwithin{figure}{section}

% ----------------------------------------------------------------

\section{ADDITIONAL INFORMATION ON EXPERIMENTS}

\subsection{Common}

\paragraph{Optimization}

In all the experiments including the one in \cref{method:example}, the objective of \cref{eq:adaptive_train} and its gradients were estimated using only one sample of $\theta_\T$ per an instance of $x$, that is, $m=1$.
We also tried larger $m$ up to $m=100$ per a single $x$, but the final improvement of performance was marginal.
We used the \texttt{AdamW} optimizer for the optimization of inductive learning, \cref{eq:inductive}, transductive learning, \cref{eq:transductive}, and the training-time optimization of the proposed method, \cref{eq:adaptive_train}.
The parameters of the optimizer were set to the default values of the library unless stated otherwise.

\paragraph{Others}

\begin{itemize}[itemsep=0pt,topsep=0pt]
    \item All implementations were done with \texttt{PyTorch 1.11.0}.
    \item The NRMSEs were computed by dividing RMSEs by $y_\text{max} - y_\text{min}$.
\end{itemize}

% ----------------------------------------------------------------

\subsection{Toy Dataset (in \cref{method:example})}

\paragraph{Data}

In the example in \cref{method:example}, we generated the data by sampling $x$ from the uniform distribution on $\mathcal{X}=[-\pi,\pi]$.
We created training, validation, and test datasets, each of which was with $40$ samples of $(x,y)$.

\paragraph{Model}

In the model, $f_\D$ is a neural net with fully-connected layers (with two hidden layers of size $16$) and the leaky ReLU activation function.

\paragraph{Optimization}

We optimized the model parameters ($\theta_\T=[a,c]$ and $\theta_\D$ in the inductive and transductive learning schemes, \cref{eq:inductive,eq:transductive}; only $\theta_\D$ in the proposed method, \cref{eq:adaptive_train}) with learning rate varied from $0.01$ to $0.0001$ exponentially.
We set the mini-batch size to $10$ and ran the optimization for $2000$ epochs.
In the inductive and transductive learning schemes, the value of $\theta_\T=[a,c]$ was constrained to be $0 \leq a \leq 2$ and $-\pi \leq c \leq \pi$.
In the training-time optimization of the proposed method, \cref{eq:adaptive_train}, we sampled $\theta_\T=[a,c]$ from the uniform distribution on $0 \leq a \leq 2$ and $-\pi \leq c \leq \pi$.
The prediction-time optimization of the proposed method, \cref{eq:adaptive_test_optim}, was performed with the Adam optimizer (without weight decay) with learning rate varied from $0.01$ to $0.0001$.
We ran it for $2000$ epochs with the full batch.

\paragraph{Test prediction errors}

\begin{figure}[t]
    \centering
    \includegraphics[clip,height=4cm]{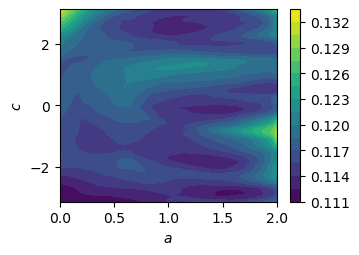}
    \caption{Test RMSEs in the result reported in \cref{fig:toy1_Rmap}.}
    \label{fig:toy1_L}
\end{figure}

In \cref{fig:toy1_Rmap} of the main text, we had no room to show the value of $\mathcal{L}$ evaluated on the test set for the same range of $\theta_\T$ as that in the plots of \cref{fig:toy1_Rmap}.
We report it here in \cref{fig:toy1_L}.
Although some extreme values of $\theta_\T=[a,c]$ tend to have slightly large $\mathcal{L}$, we can see that the model predicts similarly well with different $\theta_\T$s.

% ----------------------------------------------------------------

\subsection{Controlled Pendulum}

\paragraph{Data}

We used the dataset available\footnote{\url{https://github.com/Stable-Baselines-Team/stable-baselines/blob/f877c85b6e45084da1d8ccf73e7e730dc2001c3f/stable_baselines/gail/dataset/expert_pendulum.npz}} on the GitHub repository of the \texttt{Stable Baselines} package\footnote{\url{https://github.com/Stable-Baselines-Team/stable-baselines/}}.
It is a collection of sequences measuring the state of a frictionless compound pendulum, with a rigid uniform rod of length $1$.
The gravitational acceleration is $g=10$ in the data-generation environment.
In each sequence, the pendulum starts at a random position and is controlled to stay vertically upright by some controller.
The detail of the controller does not matter, so we just leave it unknown.
The original dataset comprises sequences of length $200$.
We disposed of the last $100$ steps of each sequence because most of the time, the pendulum is already almost stabilized at the goal state by that time.
We ran a sliding window of size $11$ on the remaining sequences and created pairs $(x,y)$ with $x$ being the first snapshot in the window and $y$ being the last $10$ snapshots in the window.
We split the original dataset into training, validation, and test sets of size $3600$, $2700$, and $2700$, respectively.

\paragraph{Model}

In the model, the neural net, $\operatorname{MLP}(t, s_t, \theta_\T, f_\T(s_t))$, is a feed-forward network with fully-connected layers with three hidden layers of size $128$.
It takes the concatenation of $s_t$, $\theta_\T$, and $f_\T(s_t)$ as the input (without using $t$ directly).
We used the \texttt{torchdiffeq} library\footnote{\url{https://github.com/rtqichen/torchdiffeq/}} for the numerical integration of the ODE with the 4th-order Runge--Kutta method.

\paragraph{Optimization}

In training, we varied the learning rate from $0.001$ to $0.00001$ exponentially, set the mini-batch size to $50$, and ran the optimization for $500$ epochs.
In the inductive and transductive learning schemes, the value of $\theta_\T$ was constrained to be in $[8,12]$.
In the training-time optimization of the proposed method, \cref{eq:adaptive_train}, we sampled $\theta_\T$ from the uniform distribution on $[8,12]$.
The prediction-time optimization of the proposed method, \cref{eq:adaptive_test_optim}, was performed simply by the grid search because $\theta_\T$ is one-dimensional.

% ----------------------------------------------------------------

\subsection{Reaction--Diffusion System}

\paragraph{Data}

We generated the data from the two-dimensional two-component reaction--diffusion system of the FitzHugh--Nagumo type:
\begin{align*}
    \frac{\partial u}{\partial t} &= 0.0015 \Delta u + u - u^3 - v + 0.005,
    \\
    \frac{\partial v}{\partial t} &= 0.005 \Delta v + u - v.
\end{align*}
The two variables of the system, $u$ and $v$, are defined over the two-dimensional space spanning $[-1,1] \times [-1,1]$.
We discretized the space with the $32 \times 32$ even grid.
In generating the data and in the trained model, the Laplacian operator, $\Delta$, was computed using the five-point stencil.
We generated $1000$ sequences of length $16$ from $t=0$ to $t=1.5$ using the 4th-order Runge--Kutta method with step size $0.001$.
We prepared training, validation, and test sets of size $400$, $300$, and $300$, respectively.

\paragraph{Model}

In the model, the neural net, $\operatorname{ConvNet}_{\theta_\T}(s_t)$, returns the sum of the outputs of two subnetworks.
One of them is a convolutional network with two hidden layers having $16$ channels and with the leaky ReLU activation function.
Another subnetwork is a two-layer convolutional network whose filters' weights are the discrete Laplacian operator (with the five-point stencil) multiplied by scalars for each filter.
Those scalars are the output of a feed-forward network that takes $\theta_\T$ as input.
It has two hidden layers of size $128$ with the leaky ReLU activation function applied to the intermediate layers and the hyperbolic tangent applied to the final layer.
When the output of this network is below $-\theta_\T/2$, it is clipped to be $-\theta_\T/2$ so that the second subnetwork does not completely cancel the output of $f_\T$.

\paragraph{Optimization}

In training, we fixed the learning rate at $0.001$, set the mini-batch size to $20$, and ran the optimization for $1000$ epochs.
In the inductive and transductive learning schemes, the value of $\theta_\T=[a,b]$ was constrained to be in $[0.001, 0.002] \times [0.001, 0.01]$.
In the training-time optimization of the proposed method, \cref{eq:adaptive_train}, we sampled $\theta_\T$ from the uniform distribution on $[0.001, 0.002] \times [0.001, 0.01]$.
The prediction-time optimization of the proposed method, \cref{eq:adaptive_test_optim}, was performed with the Adam optimizer (without weight decay) with learning rate $0.001$.
We ran it for $100$ epochs with the full batch.

% ----------------------------------------------------------------

\subsection{Predator--Prey System}

\paragraph{Data}

We used the data available online\footnote{\url{https://doi.org/10.6084/m9.figshare.10045976.v1}}, which contain measurements of the population densities of the prey (unicellular algae) and the predator, (rotifer).
In order to adjust the scale of the data values, we roughly normalized the data by multiplying $0.5$ and $0.02$ by the population densities of the two species, respectively.
We excluded the part of data containing a value larger than $7$ (after the normalization) from the training and validation sets; the test set might include such parts, but we did not check it.

\paragraph{Model}

The model has three neural networks in total: $f_\D$, the encoder of $z$, and the encoder of $\theta_\T$.
These networks have the same architecture: three hidden layers of size $128$ and the leaky ReLU activation function.

\paragraph{Optimization}

In training, we fixed the learning rate at $0.001$, set the mini-batch size to $100$, and ran the optimization for $500$ epochs.
In the inductive and transductive learning schemes, the value of $\theta_\T=[\alpha,\beta,\gamma,\delta]$ was constrained to be in $[0, 1.5] \times [0, 3] \times [0, 1.5] \times [0, 3]$.
In the training-time optimization of the proposed method, \cref{eq:adaptive_train}, we sampled $\theta_\T$ from the uniform distribution on $[0, 1.5] \times [0, 3] \times [0, 1.5] \times [0, 3]$.
The prediction-time optimization of the proposed method, \cref{eq:adaptive_test_optim}, was performed with the Adam optimizer (without weight decay) with learning rate $0.001$.
We ran it for $50$ epochs with the mini-batch of size $100$.

%% file: main.bib
@inproceedings{aletTailoringEncodingInductive2021,
  title = {Tailoring: {{Encoding}} Inductive Biases by Optimizing Unsupervised Objectives at Prediction Time},
  booktitle = {Advances in {{Neural Information Processing Systems}} 34},
  author = {Alet, Ferran and Bauza, Maria and Kawaguchi, Kenji and Kuru, Nurullah Giray and Lozano-Perez, Tomas and Kaelbling, Leslie Pack},
  date = {2021},
  pages = {29206--29217}
}

@inproceedings{arikInterpretableSequenceLearning2020,
  title = {Interpretable Sequence Learning for {{COVID-19}} Forecasting},
  booktitle = {Advances in {{Neural Information Processing Systems}} 33},
  author = {Arık, Sercan Ö. and Li, Chun-Liang and Yoon, Jinsung and Sinha, Rajarishi and Epshteyn, Arkady and Le, Long T. and Menon, Vikas and Singh, Shashank and Zhang, Leyou and Yoder, Nate and Nikoltchev, Martin and Sonthalia, Yash and Nakhost, Hootan and Kanal, Elli and Pfister, Tomas},
  date = {2020},
  pages = {18807--18818}
}

@article{azariIncorporatingPhysicalKnowledge2020,
  title = {Incorporating Physical Knowledge into Machine Learning for Planetary Space Physics},
  author = {Azari, Abigail R. and Lockhart, Jeffrey W. and Liemohn, Michael W. and Jia, Xianzhe},
  date = {2020},
  journaltitle = {Frontiers in Astronomy and Space Sciences},
  volume = {7},
  pages = {36}
}

@article{bikmukhametovCombiningMachineLearning2020,
  title = {Combining Machine Learning and Process Engineering Physics towards Enhanced Accuracy and Explainability of Data-Driven Models},
  author = {Bikmukhametov, Timur and Jäschke, Johannes},
  date = {2020},
  journaltitle = {Computers \& Chemical Engineering},
  volume = {138},
  pages = {106834}
}

@article{blasiusLongtermCyclicPersistence2020,
  title = {Long-Term Cyclic Persistence in an Experimental Predator–Prey System},
  author = {Blasius, Bernd and Rudolf, Lars and Weithoff, Guntram and Gaedke, Ursula and Fussmann, Gregor F.},
  date = {2020},
  journaltitle = {Nature},
  volume = {577},
  pages = {226--230}
}

@report{eIntegratingMachineLearning2020,
  title = {Integrating Machine Learning with Physics-Based Modeling},
  author = {E, Weinan and Han, Jiequn and Zhang, Linfeng},
  date = {2020},
  eprint = {2006.02619},
  eprinttype = {arxiv},
  archiveprefix = {arXiv}
}

@inproceedings{finnModelagnosticMetalearningFast2017,
  title = {Model-Agnostic Meta-Learning for Fast Adaptation of Deep Networks},
  booktitle = {Proceedings of the 34th {{International Conference}} on {{Machine Learning}}},
  author = {Finn, Chelsea and Abbeel, Pieter and Levine, Sergey},
  date = {2017},
  pages = {1126--1135}
}

@inproceedings{gammermanLearningTransduction1998,
  title = {Learning by Transduction},
  booktitle = {Proceedings of the 14th {{Conference}} on {{Uncertainty}} in {{Artificial Intelligence}}},
  author = {Gammerman, Alex and Vovk, Volodya and Vapnik, Vladimir},
  date = {1998},
  pages = {148--155}
}

@article{karniadakisPhysicsinformedMachineLearning2021,
  title = {Physics-Informed Machine Learning},
  author = {Karniadakis, George Em and Kevrekidis, Ioannis G. and Lu, Lu and Perdikaris, Paris and Wang, Sifan and Yang, Liu},
  date = {2021},
  journaltitle = {Nature Reviews Physics},
  volume = {3},
  pages = {422--440}
}

@misc{konUnifyingModelbasedNeural2022,
  title = {Unifying Model-Based and Neural Network Feedforward: {{Physics-guided}} Neural Networks with Linear Autoregressive Dynamics},
  author = {Kon, Johan and Bruijnen, Dennis and van de Wijdeven, Jeroen and Heertjes, Marcel and Oomen, Tom},
  options = {useprefix=true},
  date = {2022},
  number = {arXiv:2209.12489},
  eprint = {2209.12489},
  eprinttype = {arxiv},
  archiveprefix = {arXiv}
}

@inproceedings{qianIntegratingExpertODEs2021,
  title = {Integrating Expert {{ODEs}} into Neural {{ODEs}}: {{Pharmacology}} and Disease Progression},
  booktitle = {Advances in {{Neural Information Processing Systems}} 34},
  author = {Qian, Zhaozhi and Zame, William R. and Fleuren, Lucas M. and Elbers, Paul and van der Schaar, Mihaela},
  options = {useprefix=true},
  date = {2021},
  pages = {11364--11383}
}

@article{raissiPhysicsinformedNeuralNetworks2019,
  title = {Physics-Informed Neural Networks: {{A}} Deep Learning Framework for Solving Forward and Inverse Problems Involving Nonlinear Partial Differential Equations},
  author = {Raissi, Maziar and Perdikaris, Paris and Karniadakis, George E.},
  date = {2019},
  journaltitle = {Journal of Computational Physics},
  volume = {378},
  pages = {686--707}
}

@article{reichsteinDeepLearningProcess2019,
  title = {Deep Learning and Process Understanding for Data-Driven {{Earth}} System Science},
  author = {Reichstein, Markus and Camps-Valls, Gustau and Stevens, Bjorn and Jung, Martin and Denzler, Joachim and Carvalhais, Nuno and {Prabhat}},
  date = {2019},
  journaltitle = {Nature},
  volume = {566},
  number = {7743},
  pages = {195--204}
}

@incollection{sasakiNeuralGrayboxIdentification2019,
  title = {Neural Gray-Box Identification of Nonlinear Partial Differential Equations},
  author = {Sasaki, Riku and Takeishi, Naoya and Yairi, Takehisa and Hori, Koichi},
  date = {2019},
  series = {Lecture {{Notes}} in {{Computer Science}}},
  number = {11671},
  pages = {309--321}
}

@inproceedings{schnellHalfinverseGradientsPhysical2022,
  title = {Half-Inverse Gradients for Physical Deep Learning},
  booktitle = {Proceedings of the 10th {{International Conference}} on {{Learning Representations}}},
  author = {Schnell, Patrick and Holl, Philipp and Thuerey, Nils},
  date = {2022}
}

@inproceedings{seoControllingNeuralNetworks2021,
  title = {Controlling Neural Networks with Rule Representations},
  booktitle = {Advances in {{Neural Information Processing Systems}} 34},
  author = {Seo, Sungyong and Arik, Sercan O. and Yoon, Jinsung and Zhang, Xiang and Sohn, Kihyuk and Pfister, Tomas},
  date = {2021},
  pages = {11196--11207}
}

@article{sohlbergGreyBoxModelling2008,
  title = {Grey Box Modelling – Branches and Experiences},
  author = {Sohlberg, B. and Jacobsen, E.W.},
  date = {2008},
  journaltitle = {IFAC Proceedings Volumes},
  volume = {41},
  number = {2},
  pages = {11415--11420}
}

@inproceedings{takeishiPhysicsintegratedVariationalAutoencoders2021,
  title = {Physics-Integrated Variational Autoencoders for Robust and Interpretable Generative Modeling},
  booktitle = {Advances in {{Neural Information Processing Systems}} 34},
  author = {Takeishi, Naoya and Kalousis, Alexandros},
  date = {2021},
  pages = {14809--14821}
}

@article{vonruedenInformedMachineLearning2021,
  title = {Informed Machine Learning -- {{A}} Taxonomy and Survey of Integrating Knowledge into Learning Systems},
  author = {von Rueden, Laura and Mayer, Sebastian and Beckh, Katharina and Georgiev, Bogdan and Giesselbach, Sven and Heese, Raoul and Kirsch, Birgit and Pfrommer, Julius and Pick, Annika and Ramamurthy, Rajkumar and Walczak, Michal and Garcke, Jochen and Bauckhage, Christian and Schuecker, Jannis},
  options = {useprefix=true},
  date = {2021},
  journaltitle = {IEEE Transactions on Knowledge and Data Engineering}
}

@report{wangPhysicsguidedDeepLearning2021,
  title = {Physics-Guided Deep Learning for Dynamical Systems: {{A}} Survey},
  author = {Wang, Rui},
  date = {2021},
  eprint = {2107.01272},
  eprinttype = {arxiv},
  archiveprefix = {arXiv}
}

@report{wehenkelRobustHybridLearning2022,
  title = {Robust Hybrid Learning with Expert Augmentation},
  author = {Wehenkel, Antoine and Behrmann, Jens and Hsu, Hsiang and Sapiro, Guillermo and Louppe, Gilles and Jacobsen, Jörn-Henrik},
  date = {2022},
  eprint = {2202.03881},
  eprinttype = {arxiv},
  archiveprefix = {arXiv}
}

@report{willardIntegratingPhysicsbasedModeling2020,
  title = {Integrating Physics-Based Modeling with Machine Learning: {{A}} Survey},
  author = {Willard, Jared and Jia, Xiaowei and Xu, Shaoming and Steinbach, Michael and Kumar, Vipin},
  date = {2020},
  eprint = {2003.04919},
  eprinttype = {arxiv},
  archiveprefix = {arXiv}
}

@inproceedings{yinAugmentingPhysicalModels2021,
  title = {Augmenting Physical Models with Deep Networks for Complex Dynamics Forecasting},
  booktitle = {Proceedings of the 9th {{International Conference}} on {{Learning Representations}}},
  author = {Yin, Yuan and Le Guen, Vincent and Dona, Jérémie and de Bézenac, Emmanuel and Ayed, Ibrahim and Thome, Nicolas and Gallinari, Patrick},
  options = {useprefix=true},
  date = {2021}
}
